\documentclass{article}
\usepackage[margin=1in]{geometry}
\usepackage{authblk}
\usepackage[authoryear,round,sort&compress]{natbib}
\usepackage{amsmath,amssymb,amsthm,mathtools}
\usepackage{libertinus}
\usepackage[libertine]{newtxmath}
\usepackage[T1]{fontenc}
\usepackage{booktabs,graphicx,xcolor}
\definecolor{linkcol}{rgb}{0,0,0.55}\definecolor{citecol}{rgb}{0,0.45,0}
\PassOptionsToPackage{colorlinks=true,linkcolor=linkcol,citecolor=citecol}{hyperref}
\usepackage{hyperref}
\usepackage{url}

\newcommand{\R}{\mathbb{R}}
\newcommand{\U}{\mathcal{U}}
\newcommand{\Dx}{\Delta\chi}
\newtheorem{proposition}{Proposition}

\newtheorem{lemma}{Lemma}

\title{The Null Is the Hard Part:\\ Exact Tests for Memorization in Generative Models}

\author[1]{Sushovan Majhi}
\author[2]{Pramita Bagchi}
\affil[1]{\small Data Science, George Washington University, USA (s.majhi@gwu.edu)}
\affil[2]{\small Statistics, George Washington University, USA (pramita.bagchi@gwu.edu)}
\date{}

\begin{document}
\maketitle

\begin{abstract}
Memorization audits of generative models read similarity scores against thresholds, with no null distribution, and the conclusions they support can be wrong. By MemBench's rule, the benchmark's mitigations roughly halve Stable Diffusion's memorization; audited with false-discovery control, two thirds of the certified images are no longer detected under random prompt perturbations, five sixths under attention rescaling, and all of them under embedding optimization. The field's data-copying test, read against its own null, flags ten of twenty-four generators that reproduce nothing. We argue that for memorization the null is the hard part, and supply two. For a whole model, training and held-out images are exchangeable given its samples, and relabelling them is a permutation test, exact for any statistic when the held-out images are a random split; under it, a nearest-neighbour preference still fires on seven of those twenty-four, and a count restricted to the near-duplicate scale on none (McNemar $p=0.016$). For single images, the natural nulls fail twice, measurably: ranking an image among random images yields $596$ false discoveries among $2{,}365$ controls, and resampling independent generations makes the null three times too narrow. Calibrated against matched controls, the audit certifies $36$ of $61$ MemBench images at $5\%$ false-discovery rate, held-out controls are certified in $0.01\%$ of calibration splits, and on this benchmark two generations per image recover that count. A calibrated maximum, which reads occasional rather than typical copying, certifies $46$. As the scale-restricted statistic we recommend the small-scale mass of the Intersection Euler Characteristic Profile, which also counts distinct images copied and tests whether two models copy the same ones.
\end{abstract}

\section{Introduction}\label{sec:intro}

Diffusion models reproduce training images \citep{carlini2023extracting,somepalli2023diffusion}, and the detectors that followed---nearest-neighbour distance, copy-detection similarity \citep{pizzi2022sscd}, noise-prediction magnitude \citep{wen2024detecting}, trigger-prompt benchmarks \citep{webster2023reproducible,hong2024membench}---are almost all scores read against thresholds. A reported memorization rate has no error bar, an audit that screens thousands of images controls no false-discovery rate, and a mitigation is scored by how far a thresholded count falls. The sentence a practitioner wants to write---\emph{this model reproduces these $k$ training images}---is a judgment about a histogram.

This is an inference paper, and its thesis is that for memorization the null is the hard part: statistics are easy to come by, and the natural ways to calibrate them are wrong in ways that change conclusions. We make two nulls precise and measure what goes wrong without them. \emph{For a whole model}, the null is that the training images and matched held-out images are exchangeable given the model's samples, so relabelling which real images are ``training'' gives an exact permutation test for any statistic. Testing training use through a relabellable exchangeability is established in data-use and privacy auditing, where randomly publishing or inserting data before training enforces it \citep{huang2024general,steinke2023privacy,maini2021dataset}; here it is applied to reproduction in a model's samples, needs no marking and no retraining, and is exact by design when the held-out images are a random split. The one prior test of the same three samples, the data-copying test of \citet{meehan2020nonparametric}, calibrates against a different null---that a generated point is as likely as a held-out one to lie nearer the training set---which a model violates by concentrating where the data has mass, without seeing a single training image. \emph{For single images} there is nothing to relabel, and the null is that an image's statistic is exchangeable with those of control images generated the same way, a conformal outlier test \citep{bates2023conformal}. Both natural ways to build its calibration set fail, and we measure both.

Where topology enters is narrower, and we say so. Under the exact null, a statistic that reads average proximity to the training set rejects on models that merely overfit; a test of reproduction must be restricted to the scale at which copies live. A thresholded nearest-neighbour count does that. The statistic we recommend is the small-scale mass of the Intersection Euler Characteristic Profile \citep{kawamura2026intersection}, which is that count integrated over every threshold below the near-duplicate scale: its decision is steadier in the threshold, it counts distinct images copied rather than copies, and its $k$-fold form tests whether two models copy the same images. At the operating scale its specificity matches the count's (Section~\ref{sec:ddpm}).

\paragraph{Contributions.}
\begin{itemize}
\item \textbf{The exact null for memorization, and what the existing one gets wrong} (Section~\ref{sec:method}). Read against its normal reference, the data-copying statistic flags $10$ of $24$ generators that reproduce nothing; for a model that never sees its training set it reports copying at $Z_U=-9.8$ where the swap test gives $p=0.91$.
\item \textbf{The statistic decides what a rejection means} (Section~\ref{sec:ddpm}). On $29$ generators with ground truth, every statistic detects the $4$ that reproduce; of the $24$ that do not, the nearest-neighbour fraction rejects on $7$, the data-copying statistic on $5$, and a count restricted to the near-duplicate scale on none. Ten copies among $500$ samples, injected into a model that copies nothing, are caught by the scale-restricted statistics in every replicate; the nearest-neighbour fraction needs fifty.
\item \textbf{A per-image audit, with two failure modes measured} (Section~\ref{sec:sd}). Semantic leakage yields $596$ false discoveries among $2{,}365$ controls, and false concentration makes the null three times too narrow. Calibrated correctly, the audit certifies $36$ of $61$ images by Benjamini--Hochberg and $44$ by Storey's adaptive procedure, both with marginal false-discovery control, while held-out controls are certified in $0.01\%$ of calibration splits; on this benchmark two generations per image recover the count, matching generation counts between images and controls changes it little, and a calibrated maximum, which reads occasional rather than typical copying, certifies $46$.
\item \textbf{What a correct null changes} (Section~\ref{sec:changes}): mitigations scored with an error rate, the unit an audit must use, and shared memorization across models tested under the same null.
\end{itemize}

\subsection{Related work}\label{sec:related}

\paragraph{Detectors.} Beyond the detectors of Section~\ref{sec:intro} \citep{somepalli2023understanding}, memorization is explained mechanistically by \citet{gu2023memorization} and geometrically by \citet{ross2025geometric}, without a null.

\paragraph{Tests of training use.} The data-copying literature tests our three samples: \citet{meehan2020nonparametric} with an asymptotic rank test, \citet{bhattacharjee2023datacopying} with a pointwise definition, a detection method with guarantees and lower bounds, and \citet{alaa2022faithful} with a per-sample authenticity score. Data-use and privacy auditing obtain exact tests by randomization before training---publishing one of two marked copies at random \citep{huang2024general,huang2025instance} or inserting canaries at random \citep{steinke2023privacy}---and dataset inference compares suspect with held-out data, which must be a random split \citep{maini2021dataset,maini2024llm}; \citet{zhang2025membership} argue that without such randomization membership cannot be proved. Membership inference now reports calibrated per-example error rates \citep{carlini2022membership,duan2023diffusionmia,kong2024diffusionmia} and false-discovery control \citep{zhao2025miafdr}, and conformal outlier detection supplies the image-level inference \citep{bates2023conformal}. We differ in testing \emph{reproduction in a model's samples} rather than membership, with no marking and no retraining, and in measuring the calibration failures specific to generative audits. Deduplication \citep{kandpal2022deduplicating,webster2023reproducible}, differentially private training \citep{dockhorn2023dpdiffusion} and provable copyright protection \citep{vyas2023provable} are guarantees about models, where ours are error rates for measurements.

\paragraph{Shape of generated samples.} Precision and recall use unions of nearest-neighbour balls \citep{sajjadi2018assessing,kynkaanniemi2019improved,naeem2020reliable}; the Geometry Score \citep{khrulkov2018geometry}, the Manifold Topology Divergence \citep{barannikov2021manifold} and Euler characteristic curves \citep{dlotko2023euler} compare the shape of generated and real distributions. None of these attaches a null or separates reproduction from distributional fit; inference for topological summaries themselves, with minimax rates for persistence diagrams, is developed by \citet{bagchi2026inference}. The invariant we use and its certified use are developed in \citet{kawamura2026intersection} and \citet{majhi2026disentangle}, after the mixup barcode \citep{wagner2024mixup}.

\section{Two nulls}\label{sec:method}

\paragraph{Setting.} Let $G$ be $n_G$ samples of a generative model, $T$ a set of $m$ training images and $H$ a set of $m$ held-out images, as points of $\R^d$ (pixels or an embedding). A statistic $S(G,T,H)$ measures how much more the samples resemble $T$ than $H$. The statistics are the nearest-neighbour fraction $T_{\mathrm{NN}}$, the share of samples whose nearest real point lies in $T$; the data-copying statistic $Z_U$ of \citet{meehan2020nonparametric}; the thresholded count $T_{\mathrm{cnt}}$, samples within $r_0$ of a training image minus those within $r_0$ of a held-out one, with $r_0$ a quarter of the median nearest-neighbour spacing of $T\cup H$; an embedding similarity; the topological $T_{\mathrm{small}}$ defined below; and, in Appendix~\ref{sec:baseline}, the first interaction scale $T_{\mathrm{scale}}$.

\paragraph{The model-level null.}
\begin{proposition}[Swap test]\label{prop:swap}
Let $H_0$ be that, conditionally on $G$, the $2m$ real images $T\cup H$ are exchangeable with respect to their training/held-out labels. For any statistic $S(G,T,H)$ and $B$ relabellings $(T^\sigma,H^\sigma)$ of the pooled real points into two sets of size $m$, drawn independently and uniformly at random and independently of the data, the $p$-value $\big(1+\#\{b:S(G,T^{\sigma_b},H^{\sigma_b})\ge S(G,T,H)\}\big)/(1+B)$ satisfies $\Pr_{H_0}[p\le\alpha]\le\alpha$ for every $\alpha$, $m$, $n_G$ and $B$.
\end{proposition}
The proof is the standard exchangeability argument (\citealp[Sec.~15.2]{lehmann2005testing}; Appendix~\ref{app:swap}). Anything computed without the labels, from $G$ and the pooled real points, such as $r_0$ or a projection fitted on $G\cup T\cup H$, is unchanged by relabelling, so reusing it keeps the test exact. A model that has not memorized may prefer the \emph{distribution} of its training data, but it cannot prefer particular training images over held-out images drawn the same way. When $T$ and $H$ are a random split of one pool, $H_0$ holds by design for such a model. Every CIFAR-10 model below that was trained on a subset is therefore measured twice: with $H$ from the standard test split, and with $H$ drawn from training images it never saw, which is such a split. On Stable Diffusion no random split exists. There $H$ consists of training images the model did not memorize, and exchangeability is an assumption about how well the two sets match, the gap \citet{zhang2025membership} identify for any test of training use without randomization. We report the control arm, generations from non-memorized captions, as a diagnostic of that assumption, not a proof of it.

\paragraph{Why not the data-copying null.} \citet{meehan2020nonparametric} rank the distances to $T$ of generated against held-out points. Their null is $\Delta_T=1/2$: a generated point is as likely as a held-out one to lie nearer the training set, as it is when the model's distribution equals the data's. They read $\Delta_T<1/2$ as copying and calibrate the rank statistic against a standard normal. The test is valid for that null, and rejecting when $\Delta_T<1/2$ is what it is built to do. The difficulty is the reading, because $\Delta_T<1/2$ does not require access to $T$.
\begin{proposition}[A model that never sees its training set]\label{prop:meehan}
Let $T$ be a sample of size $m_T$ and $H$ a sample of size $m$ from $P$, independent of each other, and $G$ a sample from $Q$ independent of $(T,H)$. (i) When $m_T=m$ the null of Proposition~\ref{prop:swap} holds, so every swap test is exact. (ii) Fix $T$, let $d(x)=\min_{t\in T}\|x-t\|$ have a continuous distribution under $P$ and $Q$, and let $\Delta_T=\Pr[d(Y)>d(X)\mid T]$ for $Y\sim Q$, $X\sim P$. If $\Delta_T<1/2$, the data-copying statistic $Z_U$ computed from the $n_G$ samples of $Q$ and the $m$ held-out samples tends to $-\infty$ in probability as $n_G,m\to\infty$, so the data-copying test rejects at every fixed level with probability tending to one.
\end{proposition}
Proofs are in Appendix~\ref{app:proofs}. $\Delta_T<1/2$ says only that $Q$'s samples lie nearer $T$ than $P$'s do, which a model concentrated where $P$ has mass achieves without seeing $T$. With data $N(0,I_{10})$, $200$ training and $200$ held-out points, and a model that samples $N(0,s^2I)$ without seeing $T$, $Z_U$ reports copying at $-9.8$ for $s=0.7$ and underfitting at $+10.3$ for $s=1.3$, while the swap test gives $p=0.91$ and $0.79$; a model that copies a fifth of its samples from $T$ is caught by both ($Z_U=-3.7$; swap $p=0.0005$). The concentrated model's samples lie nearer every real point, trained on or not, and only relabelling removes that (Appendix~\ref{sec:baseline}).

\paragraph{The image-level null.} For one candidate image $x_i$ with its own generations $G_i$ there is nothing to relabel. Let $c_1,\dots,c_N$ be $N$ control images the model has not memorized, each with generations $G_{c_j}$ produced as $G_i$ was, and write $s_i=S(G_i,x_i)$ and $s_{c_j}=S(G_{c_j},c_j)$. The null is that $s_i$ is exchangeable with $s_{c_1},\dots,s_{c_N}$. Then $p_i=(1+\#\{j:s_{c_j}\ge s_i\})/(N+1)$ is a conformal $p$-value. $p$-values sharing one calibration set are positively dependent in the sense of \citet{benjamini2001control}, and at nominal level $q$, with $\pi_0$ the fraction of candidate images that are true nulls, Benjamini--Hochberg \citep{benjamini1995controlling} controls the false-discovery rate at $\pi_0 q$ and Storey's adaptive procedure \citep{storey2002direct,storey2004strong} at $q$, both marginally over the calibration set \citep[Cor.~1, Thm.~3]{bates2023conformal}. We say the audit \emph{certifies} the images it selects: the guarantee is on the selected set, whose expected fraction of non-memorized images is at most $q$, and not on any single image. Everything rests on ``produced as $G_i$ was'', and Section~\ref{sec:sd} measures two ways to get it wrong.

\begin{figure}[t]
\centering\includegraphics[width=0.74\textwidth]{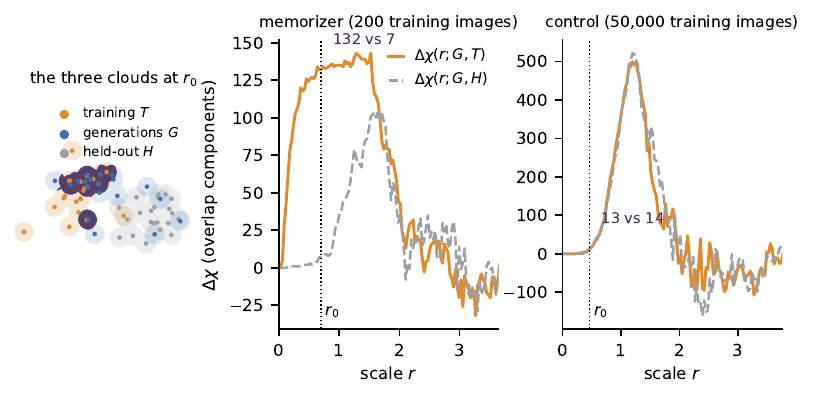}
\caption{The recommended statistic. \emph{Left:} the three clouds at the near-duplicate scale $r_0$; where a generation sits on a training image the two offsets meet, and the dark region is one blob per training image reproduced---this is what $\Dx(r_0;G,T)$ counts. \emph{Middle and right:} the measured profiles of two CIFAR-10 diffusion models. For the $200$-image memorizer the overlap with the training set has $132$ components at $r_0$ against $7$ for the held-out cloud; for the $50{,}000$-image control the two are $13$ and $14$, which is what exchangeability looks like. The swap test asks whether the gap between the two curves is larger than relabelling which real images are ``training'' would produce. (These are the swap statistic's profiles on its scale grid; the leak counts of Table~\ref{tab:sweep} are read by the readout protocol, at the near-duplicate scale of the real cloud alone, and come out at $131$ against $15$ and $20$ against $19$.)}
\label{fig:method}
\end{figure}

\paragraph{The recommended statistic.} For a finite cloud $X$ write $\U(X;r)=\bigcup_{x\in X}B(x,r)$ with $B(x,r)$ the closed ball. The Intersection Euler Characteristic Profile of two clouds,
\[
  \Dx(r;\,X,Y)\;=\;\chi\big(\U(X;r)\cap\U(Y;r)\big)\;=\;\chi_X(r)+\chi_Y(r)-\chi_{X\cup Y}(r),
\]
is an inclusion--exclusion over three Alpha complexes \citep{edelsbrunner1994alpha} (Lemma~\ref{lem:ie}) computed by one sorted sweep \citep{kawamura2026intersection}. At a scale small enough that the overlap is a union of disjoint contractible pieces, $\Dx$ counts them, one per training image with a sample on it.
\begin{proposition}[The leak count is exact under separation]\label{prop:leak}
Fix $r>0$ and suppose (i) $\|t_i-t_j\|>2r$ for every $i\neq j$, and (ii) every $g\in G$ lying within $2r$ of some $t_i$ in fact lies within $r$ of it. Then $\Dx(r;G,T)=\#\{i:\exists\,g\in G,\ \|g-t_i\|\le r\}$, the number of distinct training images that a sample reproduces to within $r$.
\end{proposition}
Hypothesis (i) is what $r_0$ targets, though $r_0$ is set from the \emph{median} spacing, so (i) fails for any two training images closer than half of it, as near-duplicate training images are; (ii) is the substantive one, and requires generations to be copies or clearly not, with nothing in between.
The statistic, illustrated in Figure~\ref{fig:method}, is
\[
  T_{\mathrm{small}}=\textstyle\int_0^{r_0}|\Dx(r;G,T)|\,dr-\int_0^{r_0}|\Dx(r;G,H)|\,dr,
\]
with $T_{\mathrm{mass}}$ the same integral over all scales. Where the hypotheses of Proposition~\ref{prop:leak} hold at every scale up to $r_0$, it is the distinct-image count at threshold $r$ integrated over $r\le r_0$. Two by-products come free: the excess component count at $r_0$, the \emph{leak count}, reads how many distinct images leaked, and the three-fold profile $\chi(\U(G_1;r)\cap\U(G_2;r)\cap\U(T;r))$ reads which images two models both reproduce. The Alpha complex costs $O(n^{\lceil d/2\rceil})$, so topological statistics are computed on a PCA projection to $d_0=5$ dimensions fitted on the pooled clouds. The restriction to small scales is what separates copying from leaning: samples that lean toward $T$ without sitting on training points leave $T_{\mathrm{small}}$ silent, while $T_{\mathrm{NN}}$ sees the lean at full strength.

\section{What the statistic decides}\label{sec:ddpm}

\begin{figure}[t]
\centering\includegraphics[width=0.70\textwidth]{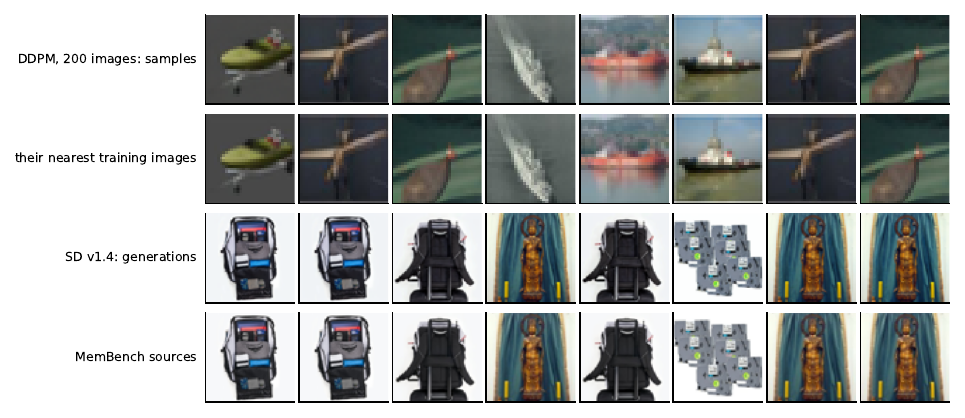}
\caption{Pixel-level reproduction. Top two rows: samples of a diffusion model trained on $200$ CIFAR-10 images, above their nearest training images. Bottom two rows: Stable Diffusion v1.4 generations from MemBench trigger prompts, above the verified memorized images they reproduce.}
\label{fig:sheet}
\end{figure}

We train small denoising diffusion models \citep{ho2020ddpm} (a 3.9M-parameter U-Net, $32\times32$, $1{,}000$-step linear schedule, DDIM sampling with $100$ steps \citep{song2021ddim}) on CIFAR-10 subsets \citep{krizhevsky2009learning} of $100$ to $5{,}000$ images and on the full $50{,}000$, for $100{,}000$ steps, with second seeds at $200$ and $1{,}000$ images, a $300{,}000$-step run at $1{,}000$, and two further seeds each at $1{,}000$, $2{,}000$ and $5{,}000$ images. Each model is sampled $500$ times and tested against its training set (capped at $2{,}000$ images) and as many held-out images, at $B=199$: CIFAR-10 test images in Table~\ref{tab:sweep}, and in Table~\ref{tab:specificity} also training images the model never saw. A sample within $r_0$ of a training image in ambient pixel space is a copy, and copies are the ground truth. Synthetic memorizers validate the statistics first (Appendix~\ref{app:synthetic}); every setting is in Appendix~\ref{app:constants}.

\begin{table}[t]
\centering\small\setlength{\tabcolsep}{5pt}
\begin{tabular}{@{}lrlrrrr@{}}
\toprule
training set & $T_{\mathrm{small}}$ & null & $p_{\mathrm{small}}$ & copies (of $500$) & distinct images (NN) & leak count \\
\midrule
$100$              & $+69.6$ & $-0.4\pm5.1$ & $0.005$ & $483$ & $96$ of $100$   & $89$ \\
$200$              & $+60.9$ & $+0.5\pm4.7$ & $0.005$ & $401$ & $128$ of $200$  & $116$ \\
$200$, seed 2      & $+65.1$ & $+0.7\pm4.8$ & $0.005$ & $394$ & $130$ of $200$  & $128$ \\
$500$              & $+11.1$ & $+0.0\pm1.8$ & $0.005$ & $103$ & $74$ of $500$   & $86$ \\
$1{,}000$          & $+0.6$  & $+0.0\pm0.5$ & $0.13$  & $2$   & $2$ of $1{,}000$ & $-6$ \\
$1{,}000$, seed 2  & $-0.4$  & $-0.0\pm0.7$ & $0.68$  & $6$   & $6$ of $1{,}000$ & $-11$ \\
$1{,}000$, $300$k steps & $+1.6$ & $-0.1\pm1.0$ & $0.055$ & $24$ & $21$ of $1{,}000$ & $19$ \\
$2{,}000$          & $+0.3$  & $+0.0\pm0.4$ & $0.32$  & $0$   & $0$ of $2{,}000$ & $10$ \\
$5{,}000$          & $+0.5$  & $-0.0\pm0.5$ & $0.17$  & $0$   & $0$ of $2{,}000$ & $4$ \\
$50{,}000$         & $-0.2$  & $-0.0\pm0.4$ & $0.66$  & $0$   & $0$ of $2{,}000$ & $1$ \\
\bottomrule
\end{tabular}
\caption{Certified memorization along the training-set size ($100$k steps unless stated; training clouds capped at $2{,}000$ images; $B=199$). Copies are samples within a quarter of the training-image spacing of a training image; the leak count is the excess overlap component count over the held-out background (Appendix~\ref{app:diag}); the nearest-neighbour tally of distinct images reproduced is its ground truth.}
\label{tab:sweep}
\end{table}

\paragraph{Detection.} The models trained on $100$ and $200$ images store their training sets, the $500$-image model copies a fifth of its samples and fires at six null standard deviations, and at $1{,}000$ images memorization stops in both seeds (Table~\ref{tab:sweep}, Figure~\ref{fig:sweep}); the full-data model is null on every statistic. Training the $1{,}000$-image model three times longer brings it to $24$ copies of $21$ images and $p=0.055$, a non-rejection with the ground truth marginal. The leak count tracks the tally of distinct images copied to within $16\%$ across the memorizing range: $401$ of the $200$-image model's $500$ samples copy $128$ images, and the count reads $116$ (Figure~\ref{fig:sheet}).

\begin{table}[t]
\centering\footnotesize\setlength{\tabcolsep}{3pt}
\begin{tabular}{@{}llccccc@{}}
\toprule
 & & detected & \multicolumn{2}{c}{rejected, test split} & \multicolumn{2}{c}{rejected, random split} \\
statistic & null & (of $4$) & diffusion ($11$) & other ($13$) & diffusion ($11$) & other ($13$) \\
\midrule
$Z_U$ & normal ref. & $4$ & $7$ & $5$ & $6$ & $4$ \\
\midrule
$T_{\mathrm{NN}}$ & swap & $4$ & $4$ & $3$ & $4$ & $3$ \\
$Z_U$ & swap & $4$ & $4$ & $1$ & $4$ & $1$ \\
$C_T$ & swap & $4$ & $4$ & $2$ & $4$ & $0$ \\
$T_{\mathrm{mass}}$ & swap & $4$ & $0$ & $4$ & $1$ & $1$ \\
$T_{\mathrm{cnt}}$ & swap & $4$ & $0$ & $0$ & $0$ & $0$ \\
$T_{\mathrm{small}}$ & swap & $4$ & $0$ & $0$ & $1$ & $0$ \\
\bottomrule
\end{tabular}
\caption{What a rejection means, on $29$ generators with ground truth: $4$ that reproduce (at least a fifth of samples within $r_0$ of a training image) and $24$ that reproduce essentially nothing ($11$ diffusion models with at most six copies; $13$ autoencoders and adversarial networks with none), at level $0.05$. Statistics as defined in Section~\ref{sec:method}; $C_T$ is the cell-weighted $Z_U$. Held-out images come from CIFAR-10's test split, or from training images the model never saw, a random split that makes the swap null hold by design ($26$ models; the three full-data models have no unused training images and keep the test split). Detection is $4$ of $4$ under both. Under the swap null every rejection of a non-reproducer is a correct rejection of that null, and none is reproduction.}
\label{tab:specificity}
\end{table}

\paragraph{Specificity.} The same code ran unmodified on thirteen variational autoencoders and adversarial networks trained on $50$ to $50{,}000$ images (Appendix~\ref{app:nondiff}), for $29$ generators in all. Four reproduce, with at least a fifth of their samples copies; twenty-four reproduce essentially nothing---the diffusion models from $1{,}000$ images up, with at most six copies, and every autoencoder and adversarial network, with none---and the $300$k-step run is set aside as marginal. Under the exact null every statistic detects the four, with either held-out set (Table~\ref{tab:specificity}). Of the twenty-four, with the random split, the nearest-neighbour fraction rejects on $7$ (four diffusion models, three others), $Z_U$ on $5$, the all-scale mass on $2$, the thresholded count on none, and $T_{\mathrm{small}}$ on one, the $5{,}000$-image diffusion model at $p=0.02$. Paired against the nearest-neighbour fraction, the count gives seven discordant pairs to none, McNemar $p=0.016$, under both designs; $T_{\mathrm{small}}$ gives seven to none with the test split and seven to one with the random split ($p=0.070$). One rejection among $24$ true nulls is what a level-$0.05$ test is expected to produce ($1.2$ on average), so the random-split count is nominal behaviour, not a loss of specificity. Those rejections are correct against the stated null, since an overfit model does depend on its training images, but none is reproduction: only copying puts samples on training images. Read against its normal reference instead, $Z_U$ flags $10$ of the $24$ ($12$ with the test split), because its mean under relabelling runs from $-10.8$ on the $200$-image diffusion models to $+14.8$ on an adversarial network---Proposition~\ref{prop:meehan} at scale. Varying the cutoff from $0.15$ to $0.5$ of the spacing on the six marginal generators changes the count's decision on three and the mass's on two, and only the mass detects the $500$-image reproducer at every cutoff (Appendix~\ref{sec:baseline}). Specificity is bought by the scale, and a count at that scale is as specific as the topological statistic or more; what the topological statistic adds is steadiness in the cutoff, the component count and the three-fold form.

\paragraph{Power.} How many copies must a model make to be seen? Starting from the $5{,}000$-image model, which copies nothing, we replace $k$ of its $500$ samples by copies of $k$ distinct training images plus Gaussian noise at $0.05$ or $0.2$ of the training-image spacing, and test against $1{,}000$ of its training images and $1{,}000$ it never saw, five replicates per cell at $B=199$. At $k=10$, two percent of the samples, $T_{\mathrm{small}}$ and the thresholded count reject in every replicate at both noise levels, and at $k=5$ $T_{\mathrm{small}}$ rejects in three of five with near-exact copies and five of five at the coarser noise, the count in four and three. The nearest-neighbour fraction and $Z_U$ reach four or five of five only at $k=50$, and the all-scale mass is unreliable throughout. Restricting the statistic to the near-duplicate scale buys power against copying as well as specificity against overfitting.

\section{Auditing single images}\label{sec:sd}

\paragraph{Data.} MemBench \citep{hong2024membench} provides $3{,}000$ prompts triggering memorized images in Stable Diffusion v1.4 \citep{rombach2022ldm}, verified by reverse image search; $61$ source images remain retrievable ($1{,}826$ prompts). Controls are $2{,}480$ live images of the most-duplicated LAION-2B captions \citep{webster2023reproducible,schuhmann2022laion}, training images of the same genre that are not memorized. Generations use DDIM, $50$ steps, guidance $7.5$ and prompt-fixed seeds: one for each of $2{,}396$ control captions, of which $2{,}365$ are stored at the $128$-pixel resolution the embedding-space audit uses, and, for nulls correlated as a memorized image's generations are, $61$ captions at the memorized counts $n_i$ and $300$ at $12$ each.

\paragraph{The model-level test first.} Against the $61$ memorized images and $61$ random controls, over five control draws at $B=199$, the $1{,}826$ generations from memorized prompts fire in every draw at $64$ and $32$ pixels ($p_{\mathrm{small}}=p_{\mathrm{mass}}=0.005$, the smallest attainable), and at $16$ and $8$ pixels ($p_{\mathrm{small}}\le0.010$); the nearest-neighbour fraction fires as well. Only $39$ generations ($2.1\%$) are pixel-level copies, of $5$ images: memorization here is mostly layout-level. The control arm is a diagnostic of the exchangeability assumption: generations from the $2{,}396$ control captions are null on $T_{\mathrm{small}}$, $T_{\mathrm{mass}}$ and $T_{\mathrm{NN}}$ at both resolutions ($p_{\mathrm{small}}$ $0.17$ to $0.66$, $p_{\mathrm{NN}}$ from $0.09$ to $0.46$), and $Z_U$ rejects there. The statistics respond differently to whatever mismatch exists between the control set and the memorized images, and one built on held-out-to-training distances is the most sensitive to it (Appendix~\ref{sec:baseline}); three non-rejections do not validate the assumption. On this model the model-level tests are exact under an assumption that one statistic contradicts, and the image-level audit below does not rest on it.

\paragraph{The image-level audit.} For memorized image $x_i$ with its $n_i$ trigger-prompt generations (median $n_i=21$), the statistic is the median SSCD similarity \citep{pizzi2022sscd} of the generations to $x_i$. Embedding space is the right place for it, since it sees layout-level copies, and the wrong place for the topological statistic, since SSCD distances concentrate (scale window $0.04$ against $2.6$ for pixels; Appendix~\ref{app:space}).

\emph{Failure one: semantic leakage.} The natural test ranks $x_i$ among random images by similarity to its own generations. A generation from a non-memorized caption ranks its own source above random images far more often than chance ($38\%$ of rank $p$-values below $0.05$), because the caption describes the image, and the design returns $596$ discoveries among $2{,}365$ controls. The valid null compares like with like: generations of non-memorized captions against their own sources (median similarity $0.052$).

\emph{Failure two: false concentration.} Resampling $n_i$ values independently from a pool of single generations of many control images gives $47$ discoveries. The pool mixes captions, so a median of values resampled from it loses the variation \emph{between} images, which remains however often one caption is generated; this happens even if generations are independent given the caption, and correlation among one caption's generations adds to it. The median of $21$ values resampled from the pool has spread (standard deviation) $0.016$, while the median of $n_i$ generations of one control caption, the correct null, has spread $0.050$. With $61$ such controls, Benjamini--Hochberg at $q=0.05$ certifies $32$, and the floor $1/62$ is what stops it; the remedy is more captions, not more generations of each (Appendix~\ref{app:floor}). With $300$ captions at $12$ generations each the audit certifies $36$ of $61$ (Figure~\ref{fig:audit}), gaining four images and losing none, and Storey's adaptive procedure certifies $44$ ($\hat\pi_0=0.29$ at $\lambda=150/301$), both covered by \citet[Cor.~1, Thm.~3]{bates2023conformal}. Held out as negatives against the other $239$, $61$ controls yield a discovery in $0.01\%$ of $20{,}000$ random splits (Appendix~\ref{app:floor}). The $25$ images this audit declines all lie below the largest control value; Storey's procedure takes eight of them, and a calibrated maximum more (below).

\paragraph{How much an audit needs.} Subsampling the committed embeddings, $R=50$ times per cell, answers what an auditor must generate (Appendix~\ref{app:floor}). One generation per image certifies $30.7\pm2.9$ images (mean and standard deviation over the subsamples), two certify $37.0\pm2.6$, and more add nothing; $100$ control captions certify $35.0\pm2.5$ where $300$ certify $36$, and held-out controls yield no discovery in any configuration. Limiting memorized images and controls alike to $12$ generations, so that both statistics are medians over the same number, certifies $35.4\pm1.3$ by Benjamini--Hochberg and $42.4$ by Storey's procedure, against $36$ and $44$. That shows the conclusions are robust to the unequal generation counts; it does not show that candidate and control statistics are exchangeable. The budgets describe this benchmark and model and are not a general prescription.

\begin{figure}[t]
\centering\includegraphics[width=0.74\textwidth]{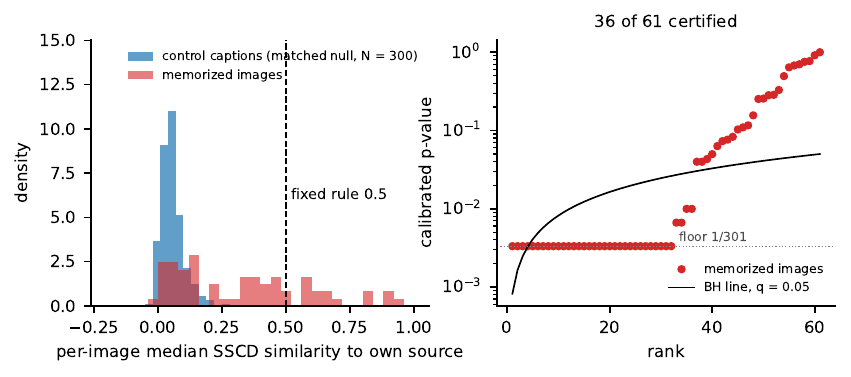}
\caption{The calibrated audit of Stable Diffusion v1.4. Left: per-image median SSCD similarity of a caption's generations to its own source image, for the $300$ control captions (the matched null) and the $61$ memorized images; the benchmark's fixed rule is the dashed line. Right: the per-image calibrated $p$-values against the Benjamini--Hochberg line at $q=0.05$; the floor at $1/301$ is the null's size.}
\label{fig:audit}
\end{figure}

\paragraph{Against the fixed rule.} The benchmark's own-source rule, a generation above similarity $0.5$, flags $39$ images, and one control generation of $2{,}365$ also clears it. Paired on the same $61$ images the two differ not in count---$30$ flagged by both, $16$ by neither, $9$ by the rule alone and $6$ by the audit alone, McNemar $p=0.61$---but in two things at once: the rule is uncalibrated, and it reads a maximum where the audit reads a median. A maximum detects occasional strong reproduction, a median typical similarity. To separate the two we calibrate the maximum itself. At a matched budget of $12$ generations per image and per control, the maximum similarity ranked against the controls' maxima certifies $46.3\pm1.3$ of $61$ over $200$ subsamples, against $35.4\pm1.3$ for the median at the same budget; it takes all $6$ of the audit's own images, about $7$ of the rule's $9$ and $3$ of the $16$ that neither flags. The rule's nine, with median similarity $0.122$ and a median $7\%$ of generations above $0.5$, are therefore mostly real occasional copies. What the rule lacks is an error rate and a matched budget, not the right statistic: calibrated, the maximum audits occasional copying and the median typical copying. We report the median throughout and the calibrated maximum as its complement.

\section{What a correct null changes}\label{sec:changes}

\begin{figure}[t]
\centering\includegraphics[width=0.78\textwidth]{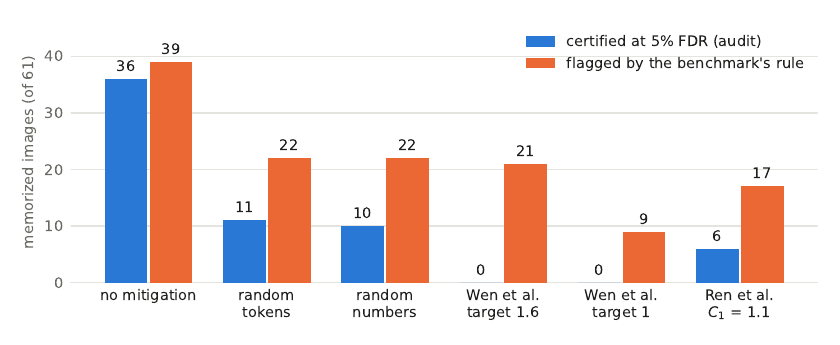}
\caption{Scoring MemBench's mitigations on Stable Diffusion v1.4. Blue: memorized images the audit certifies at $5\%$ false-discovery rate, each arm against a control null generated under the same mitigation. Orange: images the benchmark's rule flags, a generation above similarity $0.5$. By the rule, mitigation roughly halves the flagged images; under the audit, two thirds to all of the certified images are no longer detected.}
\label{fig:mitigation}
\end{figure}

\paragraph{Scoring a mitigation.} MemBench exists to score mitigations. We regenerated the trigger prompts and the $300$-caption null under each of its mitigations---four random tokens or numbers inserted into the prompt \citep{somepalli2023understanding}, the embedding optimization of \citet{wen2024detecting} at the benchmark's target and a stronger one, and the cross-attention rescaling of \citet{ren2024unveiling}---and audited each (Figure~\ref{fig:mitigation}, Appendix~\ref{app:mitigation}). Of the $36$ certified images, the audit still certifies $11$ under random tokens, $10$ under random numbers, $6$ under Ren et al.'s attention rescaling, and none under Wen's method at either target; no image is gained. A lost certification is not itself evidence of a decrease, so we test the decrease directly: the median own-similarity of every one of the $36$ falls under every mitigation (Wilcoxon signed-rank $p=1.5\times10^{-11}$ in each arm, median falls of $0.26$ to $0.43$), while the same control captions regenerated under the mitigation move by less than $0.01$ (difference in differences, Mann--Whitney $p<10^{-19}$ in every arm). The two prompt perturbations keep $9$ of the same images, whose similarity beforehand matched that of the images lost, so under these two perturbations survival is not explained by prior strength. The pixel-level test rejects in $1$, $5$, $2$, $1$ and $1$ of five control draws under random tokens, random numbers, the two Wen targets and Ren's rescaling, against $5$ of $5$ before. The benchmark's rule still flags $22$, $22$, $21$, $9$ and $17$ images, on a median $4.5$ to $15\%$ of their generations. Zero certifications mean that the audit no longer detects typical reproduction at this level, not that occasional copies are gone; those are the province of the calibrated maximum of Section~\ref{sec:sd}.

\paragraph{The unit of an audit.} Stable Diffusion~2 copies \emph{templates}, a layout shared by many training images \citep{webster2023reproducible}. With $164$ source images recovered by caption matching for the $219$ memorized prompts of \citet{ren2024unveiling}, the own-source audit certifies none, although $1{,}184$ of $1{,}576$ generations copy \emph{some} candidate. Keyed to the template instead---each prompt's generations against all candidates, with the null computed the same way---the audit selects $189$ of $197$ prompts, and $27$ of $56$ template families against family-specific nulls, at the $5\%$ level, under an exchangeability that a generic-caption null does not enforce (Appendix~\ref{app:sd2}). An audit keyed to the wrong unit certifies nothing, and a caption-matched candidate list is a hypothesis that the audit verifies.

\paragraph{Shared memorization.} Do two models reproduce the \emph{same} images? Given how many images each copies, the overlap expected by chance is hypergeometric: the two diffusion seeds trained on the same $200$ images share $109$ copied images against $83$ expected, $p=2\times10^{-15}$. The three-fold profile gives a test that is exact under the global null that training and held-out images are exchangeable given both models. Against the narrower hypothesis of no \emph{shared} copies it is a statistic whose selectivity we check empirically: it fires on that pair at $11.5$ null standard deviations and is null ($p=0.23$) on two memorizers trained on nearly disjoint subsets, where both pairwise tests fire (Appendix~\ref{app:kfold}). Stable Diffusion v1.4 and v1.5 share $36$ of the $39$ and $40$ images the rule flags for each, and the audit certifies $36$ and $35$ of $61$ against each model's own matched null, $32$ in common; RealisticVision 5.1, a community fine-tune of v1.5, certifies $30$, of which $29$ are among v1.4's $36$ (Appendix~\ref{app:v15}). Memorization persists across the seeds and related checkpoints examined. Only the two DDPM seeds separate the data from the run: v1.5 continues v1.4's training and RealisticVision fine-tunes v1.5, so those overlaps may be inherited with the weights.

\section{Limitations and conclusion}\label{sec:conclusion}

The audit needs the training images and certifies a candidate list: \emph{discoverable} memorization, the training-prefix notion of \citet{carlini2023quantifying} as named by \citet{nasr2023scalable}, a lower bound on what is there. The model-level test is exact by design only when the held-out images are a random split; with matched controls, as on Stable Diffusion, exchangeability is an assumption for which the control arm is a diagnostic, not a guarantee, and the image-level evidence rests on one benchmark of $61$ images and three related checkpoints. The image-level guarantee is marginal over the calibration set, with a $1/(N+1)$ floor; calibration-conditional $p$-values \citep{bates2023conformal} would give a stronger guarantee at a cost in power. A statistic restricted to near-duplicates is silent on resemblance, which an embedding similarity sees at the price of the semantic calibration measured here. The topological statistic costs one Alpha-complex sweep per relabelling, capping clouds at $2{,}000$ images; its readouts are diffusion-specific (Appendix~\ref{sec:leakscale}); and its gain over a thresholded count is steadiness, counting and a $k$-fold form, not specificity.

The null makes a statistic a test; the near-duplicate scale makes a test one of reproduction rather than of overfitting; and measuring semantic leakage and false concentration is what makes an image-level null correct.

\section*{Reproducibility statement}
All experiments run from one code base: the invariant is computed with the GUDHI Alpha complex, the DDPMs are a $3.9$M-parameter U-Net trained on fixed CIFAR-10 subsets with recorded indices, Stable Diffusion generations use fixed per-prompt seeds, and every $p$-value in the paper is reproduced by a single command per table. The MemBench prompts and image URLs and the LAION control table are public; we release the fetch scripts, the image identifiers, the Stable Diffusion generations at $32$ pixels, and all measurement outputs. The code, the fetch scripts, and every measurement output behind the numbers above are released at \url{https://github.com/sushovan4/memorization-audits} ($166$ measurement records, $247$ files in all); the generated arrays that regenerate deterministically from the fixed seeds are not shipped, and the scripts that regenerate them are.

\section*{Ethics statement}
The memorized images in the audit are training images that a public model reproduces; they were located by the benchmark's authors through reverse image search and include product photographs and a small number of portraits. We do not redistribute them, only their public URLs and low-resolution generated arrays, and the contact sheet shows images already published by the benchmark. The audit is a tool for finding such reproduction, which is the harm it is meant to expose. It is also, like every memorization detector, dual-use: a procedure that confirms which training images a model reproduces is a step in extracting them, and the readouts name the images. We judge the balance favourable because the method needs the training set, so it tells an owner what an outsider could already obtain by extraction, and because the alternative---releasing models with unmeasured reproduction---is the status quo the literature was written against. The LAION control captions are fetched by public URL and we redistribute no images.

\bibliographystyle{plainnat}
\bibliography{refs}

\begin{thebibliography}{55}
\providecommand{\natexlab}[1]{#1}
\providecommand{\url}[1]{\texttt{#1}}
\expandafter\ifx\csname urlstyle\endcsname\relax
  \providecommand{\doi}[1]{doi: #1}\else
  \providecommand{\doi}{doi: \begingroup \urlstyle{rm}\Url}\fi

\bibitem[Alaa et~al.(2022)Alaa, van Breugel, Saveliev, and van~der
  Schaar]{alaa2022faithful}
Ahmed Alaa, Boris van Breugel, Evgeny~S. Saveliev, and Mihaela van~der Schaar.
\newblock How faithful is your synthetic data? sample-level metrics for
  evaluating and auditing generative models.
\newblock In \emph{International Conference on Machine Learning (ICML)}, pages
  290--306, 2022.

\bibitem[Bagchi et~al.(2026)Bagchi, Majhi, Mitra, and
  Virk]{bagchi2026inference}
Pramita Bagchi, Sushovan Majhi, Atish Mitra, and {\v Z}iga Virk.
\newblock Statistical inference for persistence diagrams via landmark
  embeddings: minimax theory and finite approximation, 2026.
\newblock arXiv:2609.07691.

\bibitem[Barannikov et~al.(2021)Barannikov, Trofimov, Sotnikov, Trimbach,
  Korotin, Filippov, and Burnaev]{barannikov2021manifold}
Serguei Barannikov, Ilya Trofimov, Grigorii Sotnikov, Ekaterina Trimbach,
  Alexander Korotin, Alexander Filippov, and Evgeny Burnaev.
\newblock Manifold topology divergence: a framework for comparing data
  manifolds.
\newblock In \emph{Advances in Neural Information Processing Systems},
  volume~34, pages 7294--7305, 2021.

\bibitem[Bates et~al.(2023)Bates, Cand{\`e}s, Lei, Romano, and
  Sesia]{bates2023conformal}
Stephen Bates, Emmanuel Cand{\`e}s, Lihua Lei, Yaniv Romano, and Matteo Sesia.
\newblock Testing for outliers with conformal $p$-values.
\newblock \emph{The Annals of Statistics}, 51\penalty0 (1):\penalty0 149--178,
  2023.

\bibitem[Benjamini and Hochberg(1995)]{benjamini1995controlling}
Yoav Benjamini and Yosef Hochberg.
\newblock Controlling the false discovery rate: a practical and powerful
  approach to multiple testing.
\newblock \emph{Journal of the Royal Statistical Society: Series B},
  57\penalty0 (1):\penalty0 289--300, 1995.

\bibitem[Benjamini and Yekutieli(2001)]{benjamini2001control}
Yoav Benjamini and Daniel Yekutieli.
\newblock The control of the false discovery rate in multiple testing under
  dependency.
\newblock \emph{The Annals of Statistics}, 29\penalty0 (4):\penalty0
  1165--1188, 2001.

\bibitem[Bhattacharjee et~al.(2023)Bhattacharjee, Dasgupta, and
  Chaudhuri]{bhattacharjee2023datacopying}
Robi Bhattacharjee, Sanjoy Dasgupta, and Kamalika Chaudhuri.
\newblock Data-copying in generative models: a formal framework.
\newblock In \emph{International Conference on Machine Learning (ICML)}, pages
  2364--2396, 2023.

\bibitem[Carlini et~al.(2022)Carlini, Chien, Nasr, Song, Terzis, and
  Tram{\`e}r]{carlini2022membership}
Nicholas Carlini, Steve Chien, Milad Nasr, Shuang Song, Andreas Terzis, and
  Florian Tram{\`e}r.
\newblock Membership inference attacks from first principles.
\newblock In \emph{IEEE Symposium on Security and Privacy}, pages 1897--1914,
  2022.

\bibitem[Carlini et~al.(2023{\natexlab{a}})Carlini, Hayes, Nasr, Jagielski,
  Sehwag, Tram{\`e}r, Balle, Ippolito, and Wallace]{carlini2023extracting}
Nicholas Carlini, Jamie Hayes, Milad Nasr, Matthew Jagielski, Vikash Sehwag,
  Florian Tram{\`e}r, Borja Balle, Daphne Ippolito, and Eric Wallace.
\newblock Extracting training data from diffusion models.
\newblock In \emph{32nd USENIX Security Symposium}, pages 5253--5270,
  2023{\natexlab{a}}.

\bibitem[Carlini et~al.(2023{\natexlab{b}})Carlini, Ippolito, Jagielski, Lee,
  Tram{\`e}r, and Zhang]{carlini2023quantifying}
Nicholas Carlini, Daphne Ippolito, Matthew Jagielski, Katherine Lee, Florian
  Tram{\`e}r, and Chiyuan Zhang.
\newblock Quantifying memorization across neural language models.
\newblock In \emph{International Conference on Learning Representations
  (ICLR)}, 2023{\natexlab{b}}.

\bibitem[D{\l}otko and Gurnari(2023)]{dlotko2023euler}
Pawe{\l} D{\l}otko and Davide Gurnari.
\newblock {Euler} characteristic curves and profiles: a stable shape invariant
  for big data problems.
\newblock \emph{GigaScience}, 12:\penalty0 giad094, 2023.

\bibitem[Dockhorn et~al.(2023)Dockhorn, Cao, Vahdat, and
  Kreis]{dockhorn2023dpdiffusion}
Tim Dockhorn, Tianshi Cao, Arash Vahdat, and Karsten Kreis.
\newblock Differentially private diffusion models.
\newblock \emph{Transactions on Machine Learning Research}, 2023.

\bibitem[Duan et~al.(2023)Duan, Kong, Wang, Shi, and Xu]{duan2023diffusionmia}
Jinhao Duan, Fei Kong, Shiqi Wang, Xiaoshuang Shi, and Kaidi Xu.
\newblock Are diffusion models vulnerable to membership inference attacks?
\newblock In \emph{International Conference on Machine Learning (ICML)}, pages
  8717--8730, 2023.

\bibitem[Edelsbrunner(1995)]{edelsbrunner1995union}
Herbert Edelsbrunner.
\newblock The union of balls and its dual shape.
\newblock \emph{Discrete \& Computational Geometry}, 13\penalty0
  (3--4):\penalty0 415--440, 1995.

\bibitem[Edelsbrunner and Harer(2010)]{edelsbrunner2010computational}
Herbert Edelsbrunner and John~L. Harer.
\newblock \emph{Computational Topology: An Introduction}.
\newblock American Mathematical Society, 2010.

\bibitem[Edelsbrunner and M{\"u}cke(1994)]{edelsbrunner1994alpha}
Herbert Edelsbrunner and Ernst~P. M{\"u}cke.
\newblock Three-dimensional alpha shapes.
\newblock \emph{ACM Transactions on Graphics}, 13\penalty0 (1):\penalty0
  43--72, 1994.

\bibitem[Gu et~al.(2025)Gu, Du, Pang, Li, Lin, and Wang]{gu2023memorization}
Xiangming Gu, Chao Du, Tianyu Pang, Chongxuan Li, Min Lin, and Ye~Wang.
\newblock On memorization in diffusion models.
\newblock \emph{Transactions on Machine Learning Research}, 2025.
\newblock arXiv:2310.02664.

\bibitem[Ho et~al.(2020)Ho, Jain, and Abbeel]{ho2020ddpm}
Jonathan Ho, Ajay Jain, and Pieter Abbeel.
\newblock Denoising diffusion probabilistic models.
\newblock In \emph{Advances in Neural Information Processing Systems},
  volume~33, pages 6840--6851, 2020.

\bibitem[Hong et~al.(2025)Hong, Oh, and Sung]{hong2024membench}
Chunsan Hong, Tae-Hyun Oh, and Minhyuk Sung.
\newblock {MemBench}: Memorized image trigger prompt dataset for diffusion
  models.
\newblock \emph{Transactions on Machine Learning Research}, 2025.
\newblock arXiv:2407.17095.

\bibitem[Huang et~al.(2024)Huang, Gong, and Reiter]{huang2024general}
Zonghao Huang, Neil~Zhenqiang Gong, and Michael~K. Reiter.
\newblock A general framework for data-use auditing of {ML} models.
\newblock In \emph{ACM SIGSAC Conference on Computer and Communications
  Security (CCS)}, pages 1300--1314, 2024.

\bibitem[Huang et~al.(2025)Huang, Gong, and Reiter]{huang2025instance}
Zonghao Huang, Neil~Zhenqiang Gong, and Michael~K. Reiter.
\newblock Instance-level data-use auditing of visual {ML} models.
\newblock \emph{arXiv preprint arXiv:2503.22413}, 2025.

\bibitem[Kadkhodaie et~al.(2024)Kadkhodaie, Guth, Simoncelli, and
  Mallat]{kadkhodaie2024}
Zahra Kadkhodaie, Florentin Guth, Eero~P. Simoncelli, and St{\'e}phane Mallat.
\newblock Generalization in diffusion models arises from geometry-adaptive
  harmonic representations.
\newblock In \emph{International Conference on Learning Representations
  (ICLR)}, 2024.

\bibitem[Kamb and Ganguli(2025)]{kamb2024}
Mason Kamb and Surya Ganguli.
\newblock An analytic theory of creativity in convolutional diffusion models.
\newblock In \emph{International Conference on Machine Learning (ICML)}, pages
  28795--28831, 2025.
\newblock arXiv:2412.20292.

\bibitem[Kandpal et~al.(2022)Kandpal, Wallace, and
  Raffel]{kandpal2022deduplicating}
Nikhil Kandpal, Eric Wallace, and Colin Raffel.
\newblock Deduplicating training data mitigates privacy risks in language
  models.
\newblock In \emph{International Conference on Machine Learning (ICML)}, pages
  10697--10707, 2022.

\bibitem[Kawamura et~al.(2026)Kawamura, Majhi, and
  Mitra]{kawamura2026intersection}
Kazuhiro Kawamura, Sushovan Majhi, and Atish Mitra.
\newblock The {Intersection Euler Characteristic Profile}: {Euler} calculus and
  stability for topological interaction of ball unions, 2026.
\newblock arXiv:2608.06180.

\bibitem[Khrulkov and Oseledets(2018)]{khrulkov2018geometry}
Valentin Khrulkov and Ivan Oseledets.
\newblock Geometry score: A method for comparing generative adversarial
  networks.
\newblock In \emph{International Conference on Machine Learning}, pages
  2621--2629. PMLR, 2018.

\bibitem[Kong et~al.(2024)Kong, Duan, Ma, Shen, Zhu, Shi, and
  Xu]{kong2024diffusionmia}
Fei Kong, Jinhao Duan, RuiPeng Ma, Heng~Tao Shen, Xiaofeng Zhu, Xiaoshuang Shi,
  and Kaidi Xu.
\newblock An efficient membership inference attack for the diffusion model by
  proximal initialization.
\newblock In \emph{International Conference on Learning Representations
  (ICLR)}, 2024.

\bibitem[Krizhevsky(2009)]{krizhevsky2009learning}
Alex Krizhevsky.
\newblock Learning multiple layers of features from tiny images.
\newblock Technical report, University of Toronto, 2009.

\bibitem[Kynk{\"a}{\"a}nniemi et~al.(2019)Kynk{\"a}{\"a}nniemi, Karras, Laine,
  Lehtinen, and Aila]{kynkaanniemi2019improved}
Tuomas Kynk{\"a}{\"a}nniemi, Tero Karras, Samuli Laine, Jaakko Lehtinen, and
  Timo Aila.
\newblock Improved precision and recall metric for assessing generative models.
\newblock In \emph{Advances in Neural Information Processing Systems},
  volume~32, 2019.

\bibitem[Lehmann and Romano(2005)]{lehmann2005testing}
Erich~L. Lehmann and Joseph~P. Romano.
\newblock \emph{Testing Statistical Hypotheses}.
\newblock Springer, 3 edition, 2005.

\bibitem[Maini et~al.(2021)Maini, Yaghini, and Papernot]{maini2021dataset}
Pratyush Maini, Mohammad Yaghini, and Nicolas Papernot.
\newblock Dataset inference: Ownership resolution in machine learning.
\newblock In \emph{International Conference on Learning Representations
  (ICLR)}, 2021.

\bibitem[Maini et~al.(2024)Maini, Jia, Papernot, and Dziedzic]{maini2024llm}
Pratyush Maini, Hengrui Jia, Nicolas Papernot, and Adam Dziedzic.
\newblock {LLM} dataset inference: Did you train on my dataset?
\newblock In \emph{Advances in Neural Information Processing Systems},
  volume~37, 2024.

\bibitem[Majhi(2026)]{majhi2026disentangle}
Sushovan Majhi.
\newblock Certified topological interaction in neural representations: exact
  tests and the statistic they require, 2026.
\newblock arXiv:2609.08561.

\bibitem[Maria et~al.(2014)Maria, Boissonnat, Glisse, and
  Yvinec]{maria2014gudhi}
Cl{\'e}ment Maria, Jean-Daniel Boissonnat, Marc Glisse, and Mariette Yvinec.
\newblock The {G}udhi library: simplicial complexes and persistent homology.
\newblock In \emph{International Congress on Mathematical Software (ICMS)},
  volume 8592 of \emph{Lecture Notes in Computer Science}, pages 167--174,
  2014.

\bibitem[Meehan et~al.(2020)Meehan, Chaudhuri, and
  Dasgupta]{meehan2020nonparametric}
Casey Meehan, Kamalika Chaudhuri, and Sanjoy Dasgupta.
\newblock A non-parametric test to detect data-copying in generative models.
\newblock In \emph{International Conference on Artificial Intelligence and
  Statistics (AISTATS)}, pages 3546--3556, 2020.

\bibitem[Naeem et~al.(2020)Naeem, Oh, Uh, Choi, and Yoo]{naeem2020reliable}
Muhammad~Ferjad Naeem, Seong~Joon Oh, Youngjung Uh, Yunjey Choi, and Jaejun
  Yoo.
\newblock Reliable fidelity and diversity metrics for generative models.
\newblock In \emph{International Conference on Machine Learning (ICML)}, pages
  7176--7185, 2020.

\bibitem[Nasr et~al.(2023)Nasr, Carlini, Hayase, Jagielski, Cooper, Ippolito,
  Choquette-Choo, Wallace, Tram{\`e}r, and Lee]{nasr2023scalable}
Milad Nasr, Nicholas Carlini, Jonathan Hayase, Matthew Jagielski, A.~Feder
  Cooper, Daphne Ippolito, Christopher~A. Choquette-Choo, Eric Wallace, Florian
  Tram{\`e}r, and Katherine Lee.
\newblock Scalable extraction of training data from (production) language
  models, 2023.
\newblock arXiv:2311.17035.

\bibitem[Pizzi et~al.(2022)Pizzi, Roy, Ravindra, Goyal, and
  Douze]{pizzi2022sscd}
Ed~Pizzi, Sreya~Dutta Roy, Sugosh~Nagavara Ravindra, Priya Goyal, and Matthijs
  Douze.
\newblock A self-supervised descriptor for image copy detection.
\newblock In \emph{IEEE/CVF Conference on Computer Vision and Pattern
  Recognition (CVPR)}, pages 14532--14542, 2022.

\bibitem[Ren et~al.(2024)Ren, Li, Zeng, Xu, Lyu, Xing, and
  Tang]{ren2024unveiling}
Jie Ren, Yaxin Li, Shenglai Zeng, Han Xu, Lingjuan Lyu, Yue Xing, and Jiliang
  Tang.
\newblock Unveiling and mitigating memorization in text-to-image diffusion
  models through cross attention.
\newblock In \emph{European Conference on Computer Vision (ECCV)}, pages
  340--356, 2024.

\bibitem[Rombach et~al.(2022)Rombach, Blattmann, Lorenz, Esser, and
  Ommer]{rombach2022ldm}
Robin Rombach, Andreas Blattmann, Dominik Lorenz, Patrick Esser, and Bj{\"o}rn
  Ommer.
\newblock High-resolution image synthesis with latent diffusion models.
\newblock In \emph{IEEE/CVF Conference on Computer Vision and Pattern
  Recognition (CVPR)}, pages 10684--10695, 2022.

\bibitem[Ross et~al.(2025)Ross, Kamkari, Wu, Hosseinzadeh, Liu, Stein,
  Cresswell, and Loaiza-Ganem]{ross2025geometric}
Brendan~Leigh Ross, Hamidreza Kamkari, Tongzi Wu, Rasa Hosseinzadeh, Zhaoyan
  Liu, George Stein, Jesse~C. Cresswell, and Gabriel Loaiza-Ganem.
\newblock A geometric framework for understanding memorization in generative
  models.
\newblock In \emph{International Conference on Learning Representations}, 2025.
\newblock arXiv:2411.00113.

\bibitem[Sajjadi et~al.(2018)Sajjadi, Bachem, Lucic, Bousquet, and
  Gelly]{sajjadi2018assessing}
Mehdi S.~M. Sajjadi, Olivier Bachem, Mario Lucic, Olivier Bousquet, and Sylvain
  Gelly.
\newblock Assessing generative models via precision and recall.
\newblock In \emph{Advances in Neural Information Processing Systems},
  volume~31, 2018.

\bibitem[Schuhmann et~al.(2022)Schuhmann, Beaumont, Vencu, Gordon, Wightman,
  Cherti, Coombes, Katta, Mullis, Wortsman, Schramowski, Kundurthy, Crowson,
  Schmidt, Kaczmarczyk, and Jitsev]{schuhmann2022laion}
Christoph Schuhmann, Romain Beaumont, Richard Vencu, Cade Gordon, Ross
  Wightman, Mehdi Cherti, Theo Coombes, Aarush Katta, Clayton Mullis, Mitchell
  Wortsman, Patrick Schramowski, Srivatsa Kundurthy, Katherine Crowson, Ludwig
  Schmidt, Robert Kaczmarczyk, and Jenia Jitsev.
\newblock {LAION-5B}: An open large-scale dataset for training next generation
  image-text models.
\newblock In \emph{Advances in Neural Information Processing Systems, Datasets
  and Benchmarks Track}, volume~35, pages 25278--25294, 2022.

\bibitem[Somepalli et~al.(2023{\natexlab{a}})Somepalli, Singla, Goldblum,
  Geiping, and Goldstein]{somepalli2023diffusion}
Gowthami Somepalli, Vasu Singla, Micah Goldblum, Jonas Geiping, and Tom
  Goldstein.
\newblock Diffusion art or digital forgery? investigating data replication in
  diffusion models.
\newblock In \emph{IEEE/CVF Conference on Computer Vision and Pattern
  Recognition (CVPR)}, pages 6048--6058, 2023{\natexlab{a}}.

\bibitem[Somepalli et~al.(2023{\natexlab{b}})Somepalli, Singla, Goldblum,
  Geiping, and Goldstein]{somepalli2023understanding}
Gowthami Somepalli, Vasu Singla, Micah Goldblum, Jonas Geiping, and Tom
  Goldstein.
\newblock Understanding and mitigating copying in diffusion models.
\newblock In \emph{Advances in Neural Information Processing Systems},
  volume~36, pages 47783--47803, 2023{\natexlab{b}}.

\bibitem[Song et~al.(2021)Song, Meng, and Ermon]{song2021ddim}
Jiaming Song, Chenlin Meng, and Stefano Ermon.
\newblock Denoising diffusion implicit models.
\newblock In \emph{International Conference on Learning Representations}, 2021.

\bibitem[Steinke et~al.(2023)Steinke, Nasr, and Jagielski]{steinke2023privacy}
Thomas Steinke, Milad Nasr, and Matthew Jagielski.
\newblock Privacy auditing with one (1) training run.
\newblock In \emph{Advances in Neural Information Processing Systems},
  volume~36, 2023.

\bibitem[Storey(2002)]{storey2002direct}
John~D. Storey.
\newblock A direct approach to false discovery rates.
\newblock \emph{Journal of the Royal Statistical Society: Series B},
  64\penalty0 (3):\penalty0 479--498, 2002.

\bibitem[Storey et~al.(2004)Storey, Taylor, and Siegmund]{storey2004strong}
John~D. Storey, Jonathan~E. Taylor, and David Siegmund.
\newblock Strong control, conservative point estimation and simultaneous
  conservative consistency of false discovery rates: a unified approach.
\newblock \emph{Journal of the Royal Statistical Society: Series B},
  66\penalty0 (1):\penalty0 187--205, 2004.

\bibitem[Vyas et~al.(2023)Vyas, Kakade, and Barak]{vyas2023provable}
Nikhil Vyas, Sham~M. Kakade, and Boaz Barak.
\newblock On provable copyright protection for generative models.
\newblock In \emph{International Conference on Machine Learning (ICML)}, pages
  35277--35299, 2023.

\bibitem[Wagner et~al.(2026)Wagner, Arustamyan, Wheeler, and
  Bubenik]{wagner2024mixup}
Hubert Wagner, Nickolas Arustamyan, Matthew Wheeler, and Peter Bubenik.
\newblock Mixup barcodes: quantifying geometric-topological interactions
  between point clouds.
\newblock In \emph{42nd International Symposium on Computational Geometry (SoCG
  2026)}, volume 367 of \emph{LIPIcs}, pages 94:1--94:19, 2026.
\newblock \doi{10.4230/LIPIcs.SoCG.2026.94}.

\bibitem[Webster(2023)]{webster2023reproducible}
Ryan Webster.
\newblock A reproducible extraction of training images from diffusion models.
\newblock \emph{arXiv preprint arXiv:2305.08694}, 2023.

\bibitem[Wen et~al.(2024)Wen, Liu, Chen, and Lyu]{wen2024detecting}
Yuxin Wen, Yuchen Liu, Chen Chen, and Lingjuan Lyu.
\newblock Detecting, explaining, and mitigating memorization in diffusion
  models.
\newblock In \emph{International Conference on Learning Representations}, 2024.

\bibitem[Zhang et~al.(2025)Zhang, Das, Kamath, and
  Tram{\`e}r]{zhang2025membership}
Jie Zhang, Debeshee Das, Gautam Kamath, and Florian Tram{\`e}r.
\newblock Membership inference attacks cannot prove that a model was trained on
  your data.
\newblock In \emph{IEEE Conference on Secure and Trustworthy Machine Learning
  (SaTML)}, 2025.

\bibitem[Zhao et~al.(2025)Zhao, Qian, Chen, and Huai]{zhao2025miafdr}
Chenxu Zhao, Wei Qian, Aobo Chen, and Mengdi Huai.
\newblock Membership inference attacks with false discovery rate control.
\newblock In \emph{IEEE/CVF International Conference on Computer Vision
  (ICCV)}, pages 1216--1227, 2025.

\end{thebibliography}

\appendix
\section{Proof of Proposition~\ref{prop:swap}}\label{app:swap}
Under $H_0$, conditionally on $G$ and on the unordered pooled set of $2m$ real images, the observed assignment of the labels ``training'' and ``held-out'' is uniform over the $\binom{2m}{m}$ assignments. The $B$ relabellings are drawn uniformly from the same set, independently of the data, so $S(G,T,H)$ and $S(G,T^{\sigma_1},H^{\sigma_1}),\dots,S(G,T^{\sigma_B},H^{\sigma_B})$ are $B+1$ conditionally i.i.d.\ draws of one function of a uniform assignment, hence exchangeable. For exchangeable $(W_0,\dots,W_B)$ the rank $p$-value $(1+\#\{b\ge1:W_b\ge W_0\})/(1+B)$ is at most $\alpha$ with probability at most $\alpha$ \citep[Sec.~15.2]{lehmann2005testing}; the $\ge$ in the count is what keeps the bound under ties, and a statistic constant across relabellings correctly returns $p=1$. Nothing constrains $m$, $n_G$ or $B$. The projection and the scale $r_0$ are functions of the unordered pooled clouds, so they are the same under every relabelling; reusing them therefore evaluates one and the same statistic $S$ on each relabelled dataset, which is all the proposition requires.

\section{Proofs of Propositions~\ref{prop:meehan} and~\ref{prop:leak}}\label{app:proofs}
\paragraph{Proof of Proposition~\ref{prop:meehan}.} Given $G$, the $2m$ real points are i.i.d.\ from $P$ and independent of $G$, so their joint law is invariant under every permutation, and in particular under relabelling which $m$ of them are $T$; Proposition~\ref{prop:swap} applies. For (ii), condition on $T$ and write $n=n_G$ and $L=nm$. The Mann--Whitney count $U=\sum_{i\le n,\,j\le m}\mathbf 1[d(Y_i)>d(X_j)]$ is the statistic of \citet{meehan2020nonparametric}, since ties between $d(Y_i)$ and $d(X_j)$ have probability zero; it has mean $L\Delta_T$, since $Y_i$ and $X_j$ are independent given $T$, and changing one $Y_i$ or one $X_j$ moves $U/L$ by at most $1/n$ or $1/m$; the bounded-differences inequality gives $\Pr[|U/L-\Delta_T|\ge t\mid T]\le2\exp(-2t^2L/(n+m))$, which is Theorem~1 of \citet{meehan2020nonparametric} with the factor $2$ of the two-sided bound, which their statement omits. Now $Z_U=(U/L-1/2)\,L/\sigma_U$ with $L/\sigma_U=\sqrt{12L/(n+m+1)}\to\infty$. Taking $t=(1/2-\Delta_T)/2$, with probability tending to one $Z_U\le-t\sqrt{12L/(n+m+1)}$, which falls below every fixed normal quantile.

The proof, and every profile in the paper, uses the Euler characteristic as a valuation on unions of balls. We state that once, with the hypotheses it needs.
\begin{lemma}[Inclusion--exclusion for ball unions]\label{lem:ie}
Let $A_1,\dots,A_k$ be finite unions of closed balls in $\R^d$. Then
\[
  \chi\Big(\bigcap_{i=1}^k A_i\Big)\;=\;\sum_{\emptyset\neq J\subseteq[k]}(-1)^{|J|+1}\,\chi\Big(\bigcup_{j\in J}A_j\Big),
\]
and if the $A_i$ are pairwise disjoint then $\chi(\bigcup_i A_i)=\sum_i\chi(A_i)$. For a finite $X\subset\R^d$, $\chi(\U(X;r))$ is the Euler characteristic of the Alpha complex of $X$ at radius $r$.
\end{lemma}
\begin{proof}
This is Theorem~2.7 and Proposition~6.1 of \citet{kawamura2026intersection}, where $\chi$ is a valuation on definable sets by Euler integration and Lemma~A.1 there gives the $k$-fold inclusion--exclusion; the disjoint case is the valuation identity with an empty intersection. For $k=2$ the identity is $\Dx(r;X,Y)=\chi_X(r)+\chi_Y(r)-\chi_{X\cup Y}(r)$, and for $k$ clouds it expresses $\chi(\bigcap_i\U(X_i;r))$ through the $2^k-1$ pooled curves $\chi(\U(\bigcup_{j\in J}X_j;r))$. The last claim is the nerve theorem for the Alpha complex, whose geometric realisation is homotopy equivalent to $\U(X;r)$ \citep{edelsbrunner1995union,edelsbrunner2010computational}, and $\chi$ is a homotopy invariant.
\end{proof}

\paragraph{Proof of Proposition~\ref{prop:leak}.} Write $P_i=\U(G;r)\cap B(t_i,r)$, so that $\U(G;r)\cap\U(T;r)=\bigcup_i P_i$. By (i) the balls $B(t_i,r)$ are pairwise disjoint, and they are at positive distance from one another; hence so are the $P_i$, and $\chi$ is additive over them. Now $P_i=\bigcup_{g\in G}\big(B(g,r)\cap B(t_i,r)\big)$, and the term indexed by $g$ is nonempty exactly when $\|g-t_i\|\le 2r$, which by (ii) forces $\|g-t_i\|\le r$ and therefore $t_i\in B(g,r)$. Every nonempty term is an intersection of two balls, hence convex, and contains $t_i$; a union of convex sets sharing the point $t_i$ is star-shaped about $t_i$ and so contractible. Thus $\chi(P_i)=1$ when $P_i\neq\emptyset$ and $0$ otherwise, and $P_i\neq\emptyset$ precisely when some $g$ satisfies $\|g-t_i\|\le r$. Summing gives the claim. Both hypotheses are needed: without (i) two training images can share a component and the count merges them, and without (ii) the sets $B(g,r)\cap B(t_i,r)$ need not contain $t_i$, the union need not be contractible, and the count acquires the fragmentation Appendix~\ref{sec:leakscale} measures.

\section{The null's size and the resolution limit}\label{app:floor}
With $N$ control images the smallest attainable $p$-value is $1/(N+1)$, while the Benjamini--Hochberg threshold at rank $k$ among $61$ tests is $0.05k/61=0.00082k$. At $N=61$ the floor is $0.016$, which exceeds that threshold at every rank below $20$, so any non-empty rejection set must contain at least $20$ hypotheses (the step-up procedure rejects everything through the largest rank whose $p$-value meets its threshold). Here $32$ images sit at the floor, the threshold at rank $32$ is $0.026$, and exactly those $32$ are certified. Enlarging the null is what moves this, and enlarging it in the right direction: the null statistic shows no monotone association with how many generations each control caption gets (Spearman $+0.09$, $p=0.50$, over $n_i$ from $1$ to $138$). That speaks to its typical value, not to its spread, which can change with the count without the level moving; the matched-count rerun below is the direct check. More captions buy resolution and more generations of each do not. At $N=300$ the floor is $0.0033$, binding only below rank $5$, and the audit certifies $36$---the original $32$ still at the floor, four mid-range images gained, none lost. Storey's adaptive procedure with $\lambda=150/301$ estimates $\hat\pi_0=0.29$ and certifies $44$. Theorem~3 of \citet{bates2023conformal} proves that this procedure controls the false-discovery rate for marginal conformal $p$-values when $\lambda=K/(N+1)$ for an integer $K$, which is why $\lambda$ sits on that grid; our $p$-values are multiples of $1/301$, so this $\lambda$ selects the same images as $\lambda=0.5$ would, and the theorem applies as stated. Its conditions are a continuous statistic and null test images exchangeable with, and independent of, the controls; the one to argue is exchangeability, since a memorized image has $n_i$ generations and a control $12$. Two observations bear on it, as robustness rather than proof: the null statistic's level shows no association with the generation count, as just stated, and the audit rerun with every memorized image limited to $12$ generations, which matches the counts, certifies $35.4\pm1.3$ and $42.4$ images over $50$ subsamples.

\paragraph{Power.} Mean and standard deviation of the certified count over $50$ subsamples (\texttt{genmem\_audit\_power.py}). In every cell a fifth of the kept controls is held out as negatives and calibrated on the rest, and no cell produces a discovery among them.
\begin{center}\small
\begin{tabular}{@{}lrrrr@{}}
\toprule
generations per image & $N=50$ controls & $100$ & $200$ & $300$ \\
\midrule
$1$ & $29.5\pm5.1$ & $29.8\pm3.2$ & $30.5\pm2.9$ & $30.7\pm2.9$ \\
$2$ & $33.5\pm7.6$ & $35.5\pm3.3$ & $36.6\pm2.3$ & $37.0\pm2.6$ \\
$4$ & $33.6\pm5.9$ & $35.5\pm2.7$ & $36.2\pm2.6$ & $35.9\pm2.4$ \\
$8$ & $33.9\pm2.7$ & $35.1\pm2.7$ & $35.1\pm1.6$ & $35.8\pm1.6$ \\
$12$ & $33.9\pm2.9$ & $35.0\pm2.5$ & $34.9\pm1.6$ & $35.4\pm1.3$ \\
\bottomrule
\end{tabular}
\end{center}

Every arm of the audit is embedded from $128$-pixel generations resized to SSCD's $288$; an earlier draft had embedded the matched null from $64$-pixel arrays, and re-embedding that arm at $128$ pixels leaves the count at $36$ with no false discovery among the $300$ controls.

\paragraph{What the audit declines.} The $25$ images the audit does not certify have median own-similarity $0.076$, against $0.455$ for the $36$ certified and $0.057$ for the control null. They sit at the $70$th percentile of that null and not one of them exceeds the largest value among the $300$ controls ($0.152$ against $0.316$), so the median-score Benjamini--Hochberg audit does not reject them. That is a statement about this statistic and this procedure: Storey's procedure takes eight of the $25$, and the calibrated maximum of Section~\ref{sec:sd} more. The realised rejection threshold at this null size is an own-similarity of $0.179$, the eighth-largest control value; the strongest declined image ($0.152$) and the weakest certified one ($0.202$) sit on either side of it. As a group, however, they are elevated above the controls (Mann--Whitney $p=0.014$), so they are not simply errors in the benchmark either. The reading is that MemBench's verified list mixes images this model reproduces strongly---the $36$, whose weakest member sits at $0.202$---with a tail whose reproduction, if it occurs at all, is too faint for the median audit at this null size.

\paragraph{The negative control.} Leave-one-out over the controls cannot serve. When each of the $300$ is ranked against the other $299$, the ordered $p$-values are exactly $k/300$, twenty times the Benjamini--Hochberg line $0.05k/300$ at every rank, so zero discoveries is automatic, and leave-one-out within any subset of $s$ controls gives $p$-values $k/s$ against the line $0.05k/s$, the same ratio. A control that can fail holds out $61$ controls as negatives, calibrates on the other $239$ and applies Benjamini--Hochberg at $q=0.05$ over the $61$. Over $20{,}000$ random splits a discovery occurs in $0.01\%$ of them (mean $0.0006$ discoveries, maximum $6$), against a nominal bound of $5\%$; the procedure is conservative here because the floor $1/240$ requires six held-out controls above every calibration value before anything is rejected. The same check on the nulls generated under each mitigation and for v1.5 and RealisticVision gives at most $0.04\%$ over $5{,}000$ splits each (\texttt{genmem\_review\_checks.py}).

\section{Stable Diffusion v1.5 and RealisticVision 5.1}\label{app:v15}
Stable Diffusion v1.5 continues v1.4's training on the same data. Generating the same $1{,}826$ trigger prompts gives the same picture. In pixel space the set-level test fires in every one of five control draws at both resolutions ($p_{\mathrm{small}}\le0.010$) while its controls are null; in embedding space the per-generation median own-source similarity is $0.175$, with $1$ of the $2{,}365$ control generations above the copy threshold. Against a matched null generated for v1.5 itself, the same $300$ control captions at $12$ generations each, the per-image audit certifies $35$ of $61$ at $5\%$ FDR with no false discovery among the $300$ controls; the i.i.d.\ null that inflates v1.4's count to $47$ gives $49$ here, an overestimate of the same kind. The memorized arm's median own-similarity is $0.235$ against $0.329$ for v1.4, over its own null, with mean $0.062$. Of the $35$, $32$ are among v1.4's $36$, and by the fixed rule the two models flag $39$ and $40$ with $36$ in common: the memorization persists across these related checkpoints, though, since v1.5 continues v1.4's training, the overlap cannot separate the data from inherited weights. At the pixel level the two models copy $5$ and $10$ images and share $5$. Here we report the tally rather than the overlap count, because the leak count does not survive at this scale; Appendix~\ref{sec:leakscale} measures why, and the reason is not cloud size but the near-absence of pixel-level copies for it to lock onto.

RealisticVision 5.1, the most-used community fine-tune of v1.5 (\texttt{SG161222/Realistic\_Vision\_V5.1\_noVAE}), generates the same $1{,}826$ trigger prompts and its own $300$-caption null with the same seeds and sampler. The audit certifies $30$ of $61$ at $5\%$ FDR with no false discovery among the $300$ controls (adaptive $31$; the fixed rule flags $33$; memorized median own-similarity $0.183$ over a null with mean $0.068$). Of the $30$, $29$ are among v1.4's $36$ and $26$ are certified for all three models: the fine-tune inherits its memorization from the base. In pixel space $5$ of its generations are copies at $64$ pixels ($9$ at $32$), and the set-level test rejects in all five control draws at both resolutions ($p_{\mathrm{small}}$ from $0.010$ to $0.035$).

\section{Stable Diffusion 2: memorization by template}\label{app:sd2}
A second model line, with a candidate list assembled differently. We audit Stable Diffusion 2 ($768$ pixels, $v$-prediction, OpenCLIP conditioning; the public verbatim mirror \texttt{sd2-community/stable-diffusion-2} of the now login-gated release) on the $219$ memorized prompts distributed by \citet{ren2024unveiling}, extracted by \citet{webster2023reproducible}, whose source images we recovered by matching the prompts to LAION-2B captions in the most-duplicated metadata of \citet{webster2023reproducible} and fetching the live URLs: $197$ prompts have a live source, $164$ distinct images after near-duplicate merging, mostly Shaw Floors product photographs. We draw eight generations per prompt (DDIM, $50$ steps, guidance $7.5$, seeds fixed by prompt and index) and a $300$-caption null at $12$ generations each, embedded at $128$ pixels as before. The own-source audit of Section~\ref{sec:sd} certifies none of the $164$ images: the memorized arm's median own-similarity is $0.030$, \emph{below} the null's $0.058$, and the fixed rule flags $22$ images on $2.5\%$ of generations. The generations are copies nonetheless: $1{,}184$ of $1{,}576$ have SSCD similarity above $0.5$ to some candidate image, and only $40$ to the one their caption was matched to. The candidates fall into $56$ families at mutual similarity $0.5$---one of $65$ images, a room photographed with different floorings, and $38$ singletons---and a trigger prompt reproduces the family, not the swatch: this is the template memorization of \citet{webster2023reproducible}, and caption matching assigns a prompt one member of its template. The audit's unit must match the model's. With the statistic \emph{median over a prompt's generations of its maximum similarity to any candidate image}, computed for control captions against the same candidate set, $189$ of $197$ prompts are selected at the $5\%$ level (memorized median $0.72$ against a null median of $0.09$ and maximum of $0.16$). Applying one function to both groups does not make their statistics exchangeable: the controls are generic captions and the candidates are mostly flooring, so a non-memorized flooring caption could sit above this null, and the false-discovery statement holds only under an exchangeability the design does not enforce. The separation is large, over four times the null's maximum, and we report it descriptively with that caveat. At the family level the statistic of a family is the median, over all generations of the prompts whose matched source lies in the family, of a generation's maximum similarity to the family's images; its null is the same quantity for each control caption against that family's images, and Benjamini--Hochberg over the $56$ families selects $27$, and $6$ of the $9$ with three or more images. Families are tested with different score functions on the same control generations, a dependence the theorem of \citet{bates2023conformal} does not cover, so the family-level count is Benjamini--Hochberg applied heuristically. On Stable Diffusion v1.4 the semantic effect measured in Section~\ref{sec:sd} is a median own-similarity of $0.05$. Caption matching is a weaker ground truth than reverse image search, and it is the audit that verifies the list. In pixel space no generation lies within a quarter of the spacing of its matched source, as the template reading predicts, yet the set-level certificate against the $164$ candidates fires in every one of five control draws at $64$ and at $32$ pixels ($p_{\mathrm{small}}=0.005$ throughout, $4.0$--$5.3$ null standard deviations): the near-duplicate scale in the five-dimensional shadow registers the template that no single assigned source matches.

\section{Synthetic protocols}\label{app:synthetic}

\paragraph{Synthetic clouds.} On CIFAR-10 pixels \citep{krizhevsky2009learning}, $m=200$ training and held-out images and $200$ ``samples'' constructed three ways: copies of the training images with $10\%$ of the nearest-neighbour spacing of Gaussian noise (a memorizer); $20\%$ copies and $80\%$ other training images (a partial memorizer); and other training images only (a non-memorizer). At $B=199$: the memorizer fires on all three statistics at $p=0.005$, the smallest attainable value, with $T_{\mathrm{small}}$ twenty null standard deviations out; the non-memorizer is null on all three ($p=0.42$, $0.56$, $0.74$); and the partial memorizer separates the statistics---the all-scale mass reaches only $p=0.08$ while the small-scale mass and the scale statistic fire at $p=0.005$. The small-scale mass is the statistic we report as primary.

\paragraph{Readouts.} Three constructions with the same protocol, each designed so that a nearest-neighbour count cannot resolve it.
\emph{Leak count.} Two sample sets with identical copy fraction and an identical nearest-neighbour fraction of $0.592$: one copies $100$ distinct training images once each, the other copies $5$ images twenty times each. The excess component count reads $92$ and $10$ (ground truth $100$ and $5$).
\emph{Fidelity.} Copies of $200$ training images at noise from $2\%$ to $80\%$ of the spacing all give a nearest-neighbour fraction of exactly $1.0$; the ambient first interaction scale rises linearly from $0.010$ to $0.383$ of the spacing.
\emph{Shared copies.} Two generators copying image sets $A$ and $B$ with $|A\cap B|\in\{0,25,50,100\}$: the pairwise excess counts stay near $100$ for both while the three-fold excess reads $5$, $28$, $49$, $94$.

\section{The baseline, in one table}\label{sec:baseline}
\begin{figure}[t]
\centering\includegraphics[width=0.74\textwidth]{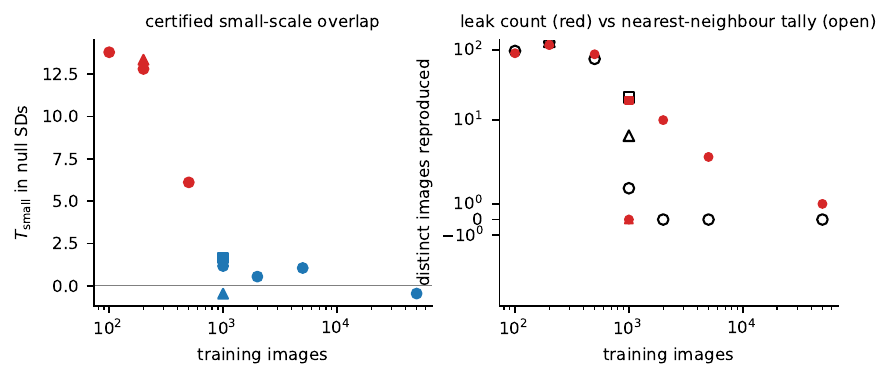}
\caption{The training-set-size sweep. Left: the certified small-scale overlap statistic in null standard deviations (red where $p_{\mathrm{small}}\le0.05$; triangles are second seeds, the square is the $300$k-step run). Right: the leak count (red) against the nearest-neighbour tally of distinct training images reproduced (open markers).}
\label{fig:sweep}
\end{figure}

Table~\ref{tab:baseline} puts every experiment of the paper in one place with every statistic under the same swap null. Where pixel-level copies exist, the nearest-neighbour fraction and the topological statistics agree. Where they do not, the two answer different questions: the nearest-neighbour fraction also fires on two of the four seeds at $1{,}000$ images, one of three at $2{,}000$, one of three at $5{,}000$ and the $300$k-step run, whose samples sit slightly closer to training images than to held-out ones on average with no sample a copy of any of them. That is a distribution-level preference for the training set, which is overfitting; the small-scale overlap statistic fires only on reproduction. The SSCD similarity is the only statistic that sees Stable Diffusion's layout-level memorization, and the topological statistics are the only ones that carry a scale, a component count, and a $k$-fold form.

\paragraph{The data-copying null, in a toy.} Data are $N(0,I_{10})$, with $200$ training and $200$ held-out points; a model samples $500$ points from $N(0,s^2I)$ and never sees the training set, so it cannot copy; a fourth model copies a fifth of its samples from the training set with small noise. $Z_U$ is read against the standard normal as \citet{meehan2020nonparametric} do, and the swap test uses the nearest-neighbour fraction with $1{,}999$ relabellings (\texttt{genmem\_toy\_null.py}).
\begin{center}\small
\begin{tabular}{@{}lrrr@{}}
\toprule
model & mean distance to $T$: generated / held-out & $Z_U$ & swap $p$ \\
\midrule
$s=0.7$, no copying & $1.77$ / $2.14$ & $-9.8$ & $0.91$ \\
$s=1.0$, no copying & $2.18$ / $2.14$ & $+0.7$ & $0.74$ \\
$s=1.3$, no copying & $2.69$ / $2.14$ & $+10.3$ & $0.79$ \\
$s=1.0$, $20\%$ copies & $1.80$ / $2.14$ & $-3.7$ & $0.0005$ \\
\bottomrule
\end{tabular}
\end{center}
The concentrated model's samples lie nearer every real point, so its distances to the training set are smaller than a held-out point's, and $Z_U$ reads that as copying at $p=4\times10^{-23}$; relabelling the real points removes it, because those samples are as near the held-out points as the training points.

\paragraph{Two held-out designs.} Every statistic was measured twice on the $26$ generators trained on subsets: with held-out images from CIFAR-10's test split, and from training-split images the model never saw (\texttt{genmem --heldout unused}; \texttt{heldout\_comparison.json}). Detection is unchanged. Decisions on non-reproducers change for five statistics: the nearest-neighbour fraction trades the $5{,}000$-image model ($p=0.016\to0.77$) for a $1{,}000$-image seed ($0.08\to0.02$); the all-scale mass loses three autoencoder and adversarial rejections and gains the $2{,}000$-image model; $C_T$ loses two adversarial rejections; $Z_U$ against its normal reference loses two; and $T_{\mathrm{small}}$ gains the $5{,}000$-image model ($p=0.17\to0.02$). The marginal $300$k-step run, with $24$ copies, moves from $p=0.055$ to $0.02$. The thresholded count and $Z_U$ under the swap null make identical decisions under both designs.

\paragraph{The data-copying test under the swap null.} \citet{meehan2020nonparametric} test the same triple with $Z_U$, the $z$-scored Mann--Whitney statistic comparing the distances of generated points to their nearest training point against the same distances for held-out points, negative when the generated distances are the smaller ones; their cell-weighted form $C_T$ averages $Z_U$ over the cells of a $k$-means partition, weighted by the held-out fraction per cell, over cells holding at least twenty generated points. We compute both on the same five-dimensional shadow and the same three clouds as every other statistic, with the partition fitted on the pooled real points so that it is invariant under relabelling ($k=5$), and with $1{,}999$ relabellings; the column $p_{Z_U}$ of Table~\ref{tab:baseline} is the result. Under the swap null $Z_U$ detects all four reproducers and rejects on five of the twenty-four negatives (three seeds at $1{,}000$ images, one at $2{,}000$ and the $50$-image adversarial network; McNemar against the small-scale statistic $p=0.06$); the cell-weighted $C_T$ detects the four and rejects on six ($p=0.03$). On Stable Diffusion, with $61$ real images per arm, its cells are too sparse and it is inconsistent across draws, so only $Z_U$ is reported there. Read against the standard-normal reference its authors use, and with the test-split held-out set of this appendix, $Z_U$ flags twelve of the twenty-four at one-sided $5\%$ (ten with the random split of Table~\ref{tab:specificity}), seven of the eleven diffusion negatives among them, four of those below $-5$, a value the original paper glosses as being, under its null, as unlikely as a standard-normal draw below $-5$. The relabelling null shows why: the mean of $Z_U$ under exchangeable labels is not zero but runs from $-10.8$ on the $200$-image diffusion models to $+14.2$ on a $200$-image adversarial network, because in pixel space a diffusion model's samples lie closer to every real image than real images lie to each other, and an adversarial network's lie farther. The asymptotic null assumes the two distance samples are exchangeable under $H_0$; they are not, and relabelling the real points is what calibrates the statistic. On Stable Diffusion, $Z_U$ under the swap null fires in every control draw on the memorized prompts of both models at both resolutions ($p\le0.0025$) and also on the control-prompt generations ($p\le0.043$ in every draw at $64$ and $32$ pixels), where the small-scale mass, the all-scale mass and the nearest-neighbour fraction are null. That rejection is a fact about the control set: generic generations sit closer to the memorized images than the LAION control images do, so relative to these generations the memorized and control clouds are not exchangeable as real clouds, and a statistic that compares held-out-to-training distances registers it. The statistics driven by generation-to-real distances do not, and the small-scale statistic cannot, since no control generation is a copy.

\paragraph{Is it the scale or the topology?} The obvious non-topological statistic at the same scale is a thresholded count: generations within $r_0$ of a training image minus those within $r_0$ of a held-out one, and its distinct-image form. Under the swap null on the same clouds both have exactly the specificity of the small-scale mass at the operating point---all four reproducers, none of the twenty-four negatives, and against the nearest-neighbour fraction the same seven discordant pairs to none, McNemar $p=0.016$---so the scale is what buys specificity, and we claim nothing else for the invariant there. What differs is the dependence on the cutoff. Where the hypotheses of Proposition~\ref{prop:leak} hold at a scale $r$ the profile at $r$ counts the distinct training images with a generation within $r$; hypothesis (ii) is scale-dependent, so this holds across all of $[0,r_0]$ only when each training image's near generations are clustered tightly enough that their $r$-balls share a point with its own. To that extent $T_{\mathrm{small}}=\int_0^{r_0}[N_T(r)-N_H(r)]\,dr$ with $N$ the distinct-image count at threshold $r$: the mass is the thresholded count integrated over every threshold below $r_0$, which is the sense in which it is scale-resolved rather than thresholded, and the stability theorem of \citet{kawamura2026intersection} is its continuous form. Varying the cutoff over $0.15$ to $0.5$ of the spacing on the six generators where either decision is marginal, the count's decision changes on three (it misses the $500$-image reproducer at $0.15$ and rejects the $1{,}000$- and $5{,}000$-image negatives at $0.5$), the mass's on two (the $5{,}000$-image negative from $0.35$ up, $p=0.050$ and $0.015$, and the marginal $300$k-step run, which it rejects from $0.35$ up), and the mass detects the $500$-image reproducer at every cutoff. Integration buys stability and power at the tight end; it is not immunity, and $r_0$ remains a choice fixed in advance.

\begin{table}[ht]
\centering\footnotesize\setlength{\tabcolsep}{2pt}
\begin{tabular}{@{}lrrrrrr@{}}
\toprule
experiment & $T_{\mathrm{small}}$ (null SDs) & $p_{\mathrm{small}}$ & $p_{\mathrm{scale}}$ & $p_{\mathrm{NN}}$ & $p_{Z_U}$ & $p_{\mathrm{SSCD}}$ \\
\midrule
synthetic memorizer & $+20.6$ & $0.0050$ & $0.0050$ & -- & -- & -- \\
synthetic 20\% memorizer & $+6.3$ & $0.0050$ & $0.0050$ & -- & -- & -- \\
synthetic non-memorizer & $-0.0$ & $0.56$ & $0.74$ & -- & -- & -- \\
DDPM, 100 images & $+13.8$ & $0.0050$ & $0.0050$ & $0.0005$ & $0.0005$ & -- \\
DDPM, 200 images & $+12.8$ & $0.0050$ & $0.0050$ & $0.0005$ & $0.0005$ & -- \\
DDPM, 200 images, seed 2 & $+13.4$ & $0.0050$ & $0.0050$ & $0.0005$ & $0.0005$ & -- \\
DDPM, 500 images & $+6.1$ & $0.0050$ & $0.0050$ & $0.0005$ & $0.0005$ & -- \\
DDPM, 1,000 images & $+1.2$ & $0.13$ & $0.49$ & $0.13$ & $0.16$ & -- \\
DDPM, 1,000 images, seed 2 & $-0.5$ & $0.68$ & $0.43$ & $0.0005$ & $0.0095$ & -- \\
DDPM, 1,000 images, seed 3 & $+0.2$ & $0.46$ & $0.77$ & $0.08$ & $0.02$ & -- \\
DDPM, 1,000 images, seed 4 & $-0.7$ & $0.77$ & $0.92$ & $0.02$ & $0.03$ & -- \\
DDPM, 1,000 images, 300k steps & $+1.6$ & $0.06$ & $0.26$ & $0.0005$ & $0.0010$ & -- \\
DDPM, 2,000 images & $+0.5$ & $0.32$ & $0.76$ & $0.67$ & $0.49$ & -- \\
DDPM, 2,000 images, seed 2 & $+0.8$ & $0.21$ & $0.27$ & $0.39$ & $0.75$ & -- \\
DDPM, 2,000 images, seed 3 & $-0.3$ & $0.61$ & $0.76$ & $0.0065$ & $0.04$ & -- \\
DDPM, 5,000 images & $+1.1$ & $0.17$ & $0.23$ & $0.02$ & $0.08$ & -- \\
DDPM, 5,000 images, seed 2 & $-2.0$ & $0.98$ & $0.97$ & $0.69$ & $0.72$ & -- \\
DDPM, 5,000 images, seed 3 & $-0.2$ & $0.57$ & $0.94$ & $0.08$ & $0.92$ & -- \\
DDPM, 50,000 images & $-0.5$ & $0.66$ & $0.74$ & $0.32$ & $0.35$ & -- \\
SD v1.4, memorized prompts, 64 px (5 draws) & $+3.6$ & $0.0050$ & $0.17$ & $0.0005$ & $0.0025$ & -- \\
SD v1.4, memorized prompts, 32 px (5 draws) & $+3.3$ & $0.0050$ & $0.31$ & $0.0005$ & $0.0025$ & -- \\
SD v1.4, memorized prompts, 16 px (5 draws) & $+3.1$ & $0.01$ & $0.31$ & $0.0010$ & -- & -- \\
SD v1.4, memorized prompts, 8 px (5 draws) & $+3.2$ & $0.01$ & $0.54$ & $0.0015$ & -- & -- \\
SD v1.5, memorized prompts, 64 px (5 draws) & $+4.0$ & $0.01$ & $0.04$ & $0.0020$ & $0.0020$ & -- \\
SD v1.5, memorized prompts, 32 px (5 draws) & $+3.9$ & $0.01$ & $0.07$ & $0.0020$ & $0.0020$ & -- \\
SD v1.4, control prompts, 64 px (5 draws) & $+0.5$ & $0.63$ & $0.72$ & $0.44$ & $0.04$ & -- \\
SD v1.4, control prompts, 32 px (5 draws) & $+0.4$ & $0.66$ & $0.72$ & $0.46$ & $0.04$ & -- \\
SD v1.4, controls vs own sources, 64 px (5 draws) & $+0.0$ & $0.95$ & $0.84$ & $0.56$ & -- & -- \\
SD v1.4, memorized prompts, SSCD space & $-1.3$ & $0.89$ & $0.48$ & $0.0005$ & -- & $0.0005$ \\
\bottomrule
\end{tabular}

\caption{Every model-level test of Sections~\ref{sec:ddpm} and~\ref{sec:sd} in one place. $T_{\mathrm{small}}$ in null standard deviations; $p_{\mathrm{scale}}$ is for the first interaction scale $T_{\mathrm{scale}}=r^\ast(G,H)-r^\ast(G,T)$, with $r^\ast(X,Y)$ half the smallest distance between the two clouds; $p$-values from $B=199$ relabellings ($1{,}999$ for the nearest-neighbour, $Z_U$ and SSCD statistics; $99$ for the topological statistics of the SSCD-space row); $p_{Z_U}$ is the data-copying statistic of \citet{meehan2020nonparametric} under the same null. Where several control draws were run, the entry is the median $z$ and the \emph{largest} $p$-value over draws. Seeds 3 and 4 at $1{,}000$ images and seeds 2 and 3 at $2{,}000$ and $5{,}000$ were measured for the specificity comparison of Section~\ref{sec:ddpm}, with the $2{,}000$- and $5{,}000$-image training clouds capped at $1{,}000$.}
\label{tab:baseline}
\end{table}

\section{Which space}\label{app:space}

The natural space for copy detection is a copy-detection embedding, and MemBench validates its prompts with SSCD \citep{pizzi2022sscd}. The swap test is exact in any space for any statistic, so SSCD's own similarity is certified under our null (Section~\ref{sec:sd}); the topological readouts are another matter, since they need distances that vary with scale and deep embeddings concentrate them. \citet{majhi2026disentangle} measure this by the \emph{scale window}---the first-percentile cross distance over the closest pair, minus one---and by the effective rank, the participation ratio $(\sum_i\lambda_i)^2/\sum_i\lambda_i^2$ of the covariance eigenvalues, and find that every route to a scale-resolved measurement fails once the window drops below about $0.1$. Table~\ref{tab:space} applies the diagnostic to the $2{,}541$ real images of the audit.

\begin{table}[ht]
\centering\small
\begin{tabular}{@{}lrrr@{}}
\toprule
space & scale window & effective rank & variance in $5$ components \\
\midrule
SSCD embedding ($1{,}024$-d) & $0.04$ & $574$ & $0.03$ \\
pixels, $32\times32$ & $2.62$ & $4.5$ & $0.64$ \\
pixels, $64\times64$ & $2.60$ & $5.1$ & $0.60$ \\
\bottomrule
\end{tabular}
\caption{Concentration diagnostic on the memorized and control images. SSCD space is in the obstructed regime; pixel space at low resolution is not.}
\label{tab:space}
\end{table}

SSCD space is fully concentrated: a window of $0.04$, an effective rank of $574$ out of $1{,}024$, and three percent of the variance in the five components a projection keeps. Pixel space at either resolution has a window twenty-five times the threshold and an effective rank near five, so a five-dimensional shadow retains most of it. The consequence is visible on the audit data. On the same three clouds in SSCD space the topological statistics are null (small-scale mass $p=0.89$, scale statistic $p=0.48$, all-scale mass $p=0.06$) while the nearest-neighbour fraction fires ($0.728$, $p=0.0005$) and the gap between the $61$ per-image medians and the $2{,}365$ control own-similarities sits $22.9$ null standard deviations out (permutation $p=10^{-4}$): a concentrated space loses the scale-resolved structure, not the signal. Hence the audit's organisation---detection with error control in whichever space the practitioner's detector lives, and the readouts in pixel space, where they are valid.

\section{Auditing the benchmark's mitigations}\label{app:mitigation}
MemBench evaluates four inference-time mitigations, and we audit all four, implemented after the benchmark's code. \emph{Random token addition} inserts four tokens drawn uniformly from the CLIP vocabulary at random positions of the prompt, \emph{random number addition} four integers below $100{,}000$ (the range in the released code): four insertions are the inference-time setting of \citet{somepalli2023understanding}, where the benchmark's own runs default to one and its paper sweeps one to six. \emph{Wen et al.} \citep{wen2024detecting} minimise, before sampling, the magnitude of the text-conditional noise prediction $\|\epsilon(x_T,e)-\epsilon(x_T,e_\emptyset)\|_2$ at the first denoising step over the prompt embedding $e$, with $e_\emptyset$ the embedding of the empty prompt, (AdamW, learning rate $0.05$, at most ten iterations, early stop below a target), and sample from the same initial latent. Their loss is a Frobenius norm over a batch of latents ($4$ in their code, $10$ in the benchmark's), so the benchmark's target of $5$ is a per-sample norm of $1.6$; we optimise one latent per prompt at per-sample targets $1.6$ and $1.0$. Matching the per-sample scale does not make the two optimisation problems the same, since a batch objective can trade one latent's norm against another's; our arm is Wen's method in single-latent form at the benchmark's per-sample level, not a reproduction of the benchmark's run. \emph{Ren et al.} \citep{ren2024unveiling} rescale the cross-attention paid to the start-of-text token by $C_1=1.1$ at every denoising step, the benchmark's setting, which draws attention away from the trigger tokens. The trigger prompts start at a median magnitude of $2.5$ against $0.9$ for control captions---the detection signal of \citet{wen2024detecting}---so $73\%$ of them are optimised at the benchmark's target and $92\%$ at the strong one, against $6\%$ and $42\%$ of control captions. Every arm, memorized and null, is regenerated under the mitigation with the original seeds, stored as $128$-pixel JPEGs, and audited exactly as in Section~\ref{sec:sd}.

\begin{table}[ht]
\centering\small
\begin{tabular}{@{}lrrrrrr@{}}
\toprule
mitigation & certified & rule & pixel copies & $p_{\mathrm{small}}$ ($64$ px) & null mean & memorized median \\
\midrule
none & $36$ of $61$ & $39$ & $39$ & $0.005$ & $0.064$ & $0.329$ \\
RTA & $11$ of $61$ & $22$ & $3$ & $0.255$ & $0.052$ & $0.056$ \\
RNA & $10$ of $61$ & $22$ & $3$ & $0.025$ & $0.058$ & $0.066$ \\
Wen, target $1.6$ & $0$ of $61$ & $21$ & $0$ & $0.130$ & $0.061$ & $0.064$ \\
Wen, target $1$ & $0$ of $61$ & $9$ & $0$ & $0.315$ & $0.048$ & $0.046$ \\
Ren, $C_1=1.1$ & $6$ of $61$ & $17$ & $1$ & $0.150$ & $0.057$ & $0.068$ \\
\bottomrule
\end{tabular}

\caption{The audit under mitigation. Certified: images at $5\%$ false-discovery rate against the matched null generated under the same mitigation ($300$ captions, $12$ generations each; held-out controls yield a discovery in at most $0.04\%$ of splits in every arm, Appendix~\ref{app:floor}). Rule: images with a generation above the benchmark's $0.5$ own-source threshold. Pixel copies: generations within a quarter of the training-image spacing of a memorized image at $64$ pixels. $p_{\mathrm{small}}$: the largest over five control draws; the text gives the number of draws that reject. Null mean: the mean of the $300$ control captions' median own-similarities; memorized median: the median over the $61$ images of the same statistic.}
\label{tab:mitigation}
\end{table}

Table~\ref{tab:mitigation} is the result. The memorized arm's median own-source similarity falls from $0.329$ to the null's level under every mitigation ($0.056$ against $0.052$ for random tokens, $0.066$ against $0.058$ for numbers, $0.064$ against $0.061$ and $0.046$ against $0.048$ for Wen's method), and no image is certified under a mitigation that was not certified before. Random tokens and random numbers keep $9$ of the same $11$ and $10$ images, whose unmitigated own-similarity (median $0.458$) matches that of the images lost ($0.451$): under these two perturbations survival is not explained by prior strength, and we claim no more than that. The decrease itself is tested directly: every one of the $36$ certified images has lower median own-similarity under every mitigation (Wilcoxon signed-rank $p=1.5\times10^{-11}$ per arm; median falls of $0.30$, $0.26$, $0.38$, $0.43$ and $0.30$ for random tokens, random numbers, the two Wen targets and Ren's rescaling), while the control captions move by $0.008$ or less (difference in differences, Mann--Whitney $p<10^{-19}$). The rule's residue is occasional exceedance: the images it flags that the audit does not certify clear the threshold on a median $7\%$, $15\%$, $8\%$, $4.5\%$ and $9.5\%$ of their generations under the five arms and have median own-similarity above the null's median but below the rejection threshold (at the $63$rd to $87$th percentile of the null); the number of generations above the threshold falls from $447$ of the $1{,}811$ embedded generations to $171$, $181$, $73$, $21$ and $114$. Under Wen's method at the benchmark's target, $21$ images are flagged and none is certified by the median audit at $5\%$ FDR. The two scores answer different questions, occasional against typical reproduction, and only one of them carries an error rate; occasional copies that survive a mitigation still matter, and a calibrated maximum is the way to count them. The pixel-space certificate moves the same way: pixel-level copies fall from $39$ to $3$, $3$, $0$, $0$ and $1$, and the set-level test, which rejected in all five control draws before, rejects in $1$ of $5$ under random tokens ($p_{\mathrm{small}}$ from $0.025$ to $0.26$ at $64$ pixels), in $2$ and $1$ of $5$ under the two Wen settings ($0.005$ to $0.13$, and $0.03$ to $0.32$), in $1$ of $5$ under Ren's rescaling ($0.03$ to $0.15$), and still in $5$ of $5$ under random numbers ($0.005$ to $0.025$); the counts are the same at $32$ pixels except under Ren's rescaling, where $3$ of $5$ draws reject.

\section{Shared memorization as a certified test}\label{app:kfold}
For two generators $G_1,G_2$ and real clouds $T,H$, the three-fold small-scale mass $S_3=\int_0^{r_0}|\chi(\U(G_1;r)\cap\U(G_2;r)\cap\U(T;r))|\,dr$ minus the same with $H$ is, under the null that $T$ and $H$ are exchangeable given $G_1$ and $G_2$, an exact test by Proposition~\ref{prop:swap} with the pair $(G_1,G_2)$ in the role of $G$; the pairwise masses $S_1,S_2$ and the tally intersection (training images with a generation of \emph{both} models within $r_0$, less the held-out count) are computed under the same relabellings, $B=199$, on the five-dimensional shadow of the four pooled clouds. Three cases on the CIFAR-10 diffusion models. \emph{Same data:} the two $200$-image seeds; $S_3$ at $53.3$ against a null of $0.3\pm4.6$ ($p=0.005$), both pairwise masses at $13$ standard deviations, the tally intersection at $84$ against $0.4\pm8.0$, and the three-fold count at $r_0$ reading $116$ against the ambient tally's $109$ shared images. \emph{Disjoint data:} the $100$- and $200$-image models, whose training subsets share one image, against $T$ = their pooled $299$ training images and $299$ held-out images; $S_1$ and $S_2$ at $10$ standard deviations ($p=0.005$), $S_3$ at $0.5$ against $0.0\pm0.4$ ($p=0.23$), the tally intersection $1$ ($p=0.50$), the three-fold count $2$; the tally finds the one shared image, which is all the two models could share: each can copy only its own training images, so a chance-overlap reference that treats the $299$ pooled images as eligible for both does not apply. Each model depends on the training set, so the pairwise null is false and both pairwise tests say so; the three-fold statistic is null because there is nothing they reproduce in common. Note what this does and does not show. The exchangeability null under which the three-fold test is exact is more restrictive than ``no shared memorization'' and already fails here, where each model memorizes its own images; the non-rejection shows that the statistic is selective for shared copies, not that the test controls error under the weaker hypothesis. \emph{Neither memorizes:} the two $1{,}000$-image seeds, whose tallies of $2$ and $6$ copies happen to share one image; $S_3$ is identically zero under every relabelling ($p=1$), the pairwise masses are null ($p=0.87$ and $0.52$), and so is the tally intersection. The tally intersection under the same null is the honest baseline and agrees on every decision; what the three-fold profile adds is the same as elsewhere in the paper, a count with no assignment step and a statistic integrated over scales.

\section{The test on two other model families}\label{app:nondiff}

\paragraph{Generators that are not diffusion models.} Nothing in the swap null mentions diffusion. It says that conditionally on $G$ and on the unordered pooled real set, ``training'' and ``held-out'' are exchangeable labels, which holds for any process that emits samples. To check that the generality is real rather than nominal we trained a variational autoencoder and an adversarial network (each with a $1.2$M-parameter decoder or generator) on the same CIFAR-10 subsets and passed their samples to the same measurement code, unmodified.

The overlap mass fires on both models trained on $200$ images ($p_{\mathrm{mass}}=0.015$ for the network, $0.025$ for the autoencoder) and is null on both trained on the full $50{,}000$ ($0.960$ and $0.855$): sensitivity and specificity outside diffusion against the null as stated. Read as a claim about reproduction those rejections are false positives, and Section~\ref{sec:ddpm} counts them as such; every statistic is in Table~\ref{tab:nondiff}. What these models do not do is \emph{copy}. Neither puts a single sample within a quarter of the training-image spacing, nor does a third attempt with an adversarial network on only $50$ images---whose samples lie nearer a training image than a held-out one $83.6\%$ of the time, so the nearest-neighbour fraction rejects at $p=0.011$ while the small-scale statistic correctly abstains at $p=0.56$. That is the cleanest case in the paper of the two statistics parting company, and Section~\ref{sec:ddpm} counts how often they do. This follows from the objectives. The exact minimiser of the denoising objective on a finite training set is the empirical distribution smoothed by the noise kernel, so a diffusion model's optimum \emph{is} its training points; an adversarial objective asks only for a distribution match, and nothing in it forces pixel-exactness even under complete mode collapse. Diffusion models are the ones that copy, which is why the theory of memorization has concentrated there~\citep{kadkhodaie2024,kamb2024}.

\begin{table}[ht]
\centering\small
\begin{tabular}{@{}llrrrrr@{}}
\toprule
model & training set & $p_{\mathrm{mass}}$ & $p_{\mathrm{small}}$ & $p_{\mathrm{NN}}$ & $p_{Z_U}$ & pixel copies \\
\midrule
GAN & $50$ & $\mathbf{0.010}$ & $0.555$ & $\mathbf{0.011}$ & $\mathbf{0.017}$ & $0/500$ \\
GAN & $100$ & $0.235$ & $0.540$ & $0.111$ & $0.390$ & $0/500$ \\
GAN & $200$ & $\mathbf{0.015}$ & $1.000$ & $0.064$ & $0.123$ & $0/500$ \\
GAN & $200$, seed 2 & $0.275$ & $1.000$ & $0.501$ & $0.052$ & $0/500$ \\
GAN & $500$ & $0.635$ & $0.990$ & $0.830$ & $0.812$ & $0/500$ \\
GAN & $1{,}000$ & $0.110$ & $0.950$ & $0.163$ & $0.263$ & $0/500$ \\
GAN & $50{,}000$ & $0.960$ & $0.485$ & $0.110$ & $0.428$ & $0/500$ \\
VAE & $200$ & $\mathbf{0.025}$ & $0.965$ & $\mathbf{0.028}$ & $0.238$ & $0/500$ \\
VAE & $200$, seed 2 & $\mathbf{0.035}$ & $0.620$ & $0.098$ & $0.167$ & $0/500$ \\
VAE & $500$ & $0.940$ & $0.675$ & $\mathbf{0.037}$ & $0.196$ & $0/500$ \\
VAE & $1{,}000$ & $0.815$ & $0.510$ & $0.281$ & $0.426$ & $0/500$ \\
VAE & $2{,}000$ & $0.635$ & $0.175$ & $0.236$ & $0.307$ & $0/500$ \\
VAE & $50{,}000$ & $0.855$ & $0.145$ & $0.281$ & $0.503$ & $0/500$ \\
\bottomrule
\end{tabular}
\caption{The same test, unmodified, on two model families it was not designed against, over training sets from $50$ to $50{,}000$ images, with the test-split held-out set. Not one of the thirteen places a sample within a quarter of the training-image spacing, so every rejection here is overfitting without reproduction: the all-scale mass rejects on four, the nearest-neighbour fraction on three and the data-copying statistic $Z_U$ of \citet{meehan2020nonparametric} on one, while the small-scale statistic rejects on none of them.}
\label{tab:nondiff}
\end{table}

\section{Constants}\label{app:constants}
\emph{Invariant.} Alpha complexes \citep[GUDHI,][]{maria2014gudhi} on the $d_0=5$ PCA shadow of the pooled clouds; profiles on a $200$-point scale grid from $0$ to the median cross-cloud distance; the near-duplicate scale $r_0$ is a quarter of the median nearest-neighbour spacing of the real points. \emph{DDPMs.} U-Net with $64$ base channels and $3.9$M parameters, $32\times32$, $1{,}000$-step linear $\beta$ schedule, Adam at $2\times10^{-4}$, batch $128$, $100{,}000$ steps ($300{,}000$ where stated), DDIM sampling with $100$ steps and $\eta=0$; $500$ samples per model. \emph{Stable Diffusion.} v1.4, v1.5 and RealisticVision 5.1 in half precision, and Stable Diffusion 2 at $768$ pixels with $v$-prediction, all with DDIM at $50$ steps, guidance $7.5$, and seeds fixed by a hash of the prompt (and the generation index where there are several); generations stored at $64\times64$ and $32\times32$ by bicubic resampling and, for the embedding-space audit, as $128$-pixel JPEGs. \emph{Images.} MemBench source images and LAION controls center-cropped to squares and resampled the same way; SSCD embeddings from the public \texttt{disc\_large} model on $288\times288$ inputs, $\ell_2$-normalised.

\section{Why the leak count is not reported for Stable Diffusion}\label{sec:leakscale}

The leak count is an excess: overlap components of $(G,T)$ at the near-duplicate scale, less the same count against a held-out cloud. On the DDPMs it tracks the nearest-neighbour tally of distinct training images reproduced to within $16\%$. On Stable Diffusion it does not, and it is worth being precise about the failure, because the obvious diagnosis is wrong.

\emph{The background is not the problem.} Computed as the paper's swap statistic is---one PCA and one near-duplicate scale $r_0$ for the pooled real cloud $T\cup H$, so that the two counts are commensurable---the count against the $61$ memorized images is $42.9\pm3.5$ against $13.1\pm4.1$ for a held-out control cloud, over $15$ draws of that cloud from the $2{,}480$ live LAION images. The excess is $29.7\pm5.5$, a standard error of $1.4$. Estimating the background better does not help: it is small and well separated from the signal, and the spread that remains is the projection's, not the control sample's.

\emph{The excess is real.} Repeating the whole construction with the $2{,}396$ control generations in place of the memorized ones---generations that reproduce nothing, with an ambient tally of $0$---gives an excess of $5.3\pm12.8$ over $15$ draws, a $95\%$ interval of $[-1.2, +11.7]$ that covers zero, against $[+26.9, +32.5]$ for the memorized generations. So the excess separates a model that reproduces from one that does not. What follows is not that the readout is noise; it is that the number it produces is not a count of images.

\emph{What the count counts.} Only $5$ memorized images have a generation within a quarter of the image spacing in ambient pixel space. In the five-dimensional shadow, $31$ lie within the near-duplicate scale of some generation, against $11$ control images---layout-level proximity, which is what this model's memorization mostly is, rather than reproduction. The count reads $43$ for those $31$, over-reporting by a further $40\%$ as components fragment. Neither number is the quantity the readout is defined to give.

\emph{It depends on how many samples were drawn.} Subsampling the generations, the count reads $4$, $11$, $18$, $19$ and $43$ at $|G|=61$, $122$, $305$, $610$ and $1{,}826$, tracking the contacted-image count ($4$, $10$, $16$, $19$, $31$) closely until the largest cloud. The ambient tally over the same subsamples moves only $0$, $2$, $4$, $4$, $5$. A count of distinct leaked images must saturate as more samples are drawn; this one does not, because at this scale in a projection most contacts are not copies.

The readout is therefore valid as a count where reproduction is near-exact, which is the DDPM regime and not this one, and we report the nearest-neighbour tally instead. The limitation is the readout's, not the test's: the set-level certificate and the per-image audit do not depend on it.

\section{Diagnostics behind the counts}\label{app:diag}
\emph{Background.} A five-dimensional shadow admits chance near-neighbours: on the DDPM sweep the overlap component count against the held-out cloud runs from $5$ to $35$, and every leak count is reported as the excess over it. \emph{Three clouds.} The near-duplicate scale must be taken from the real cloud alone; a generator cloud full of copies collapses the median spacing and reads the count far below the merge scale (we found this on the two $200$-image seeds, where the uncorrected count fell to a small fraction of the ground truth of $109$; corrected, $108$ by the readout protocol, against $116$ on the swap statistic's scale grid, the same distinction as in Figure~\ref{fig:method}). \emph{Semantic leak.} With SSCD similarity as the statistic, a control caption's generation ranks its own source image above random images at median $p=0.12$ across $2{,}365$ competitors; $38\%$ of controls fall below $0.05$ and the Benjamini--Hochberg step-up selects $596$ of them. The own-source similarities themselves are small (median $0.052$, one of $2{,}365$ above $0.5$), which is why the calibrated null of Section~\ref{sec:sd} works and the rank design does not. In pixel space the own-source design is null (Table~\ref{tab:baseline}).

\end{document}